\pdfoutput=1

\documentclass[english,10pt,journal,compsoc]{IEEEtran}
\usepackage[T1]{fontenc}
\usepackage[latin9]{inputenc}
\usepackage{calc}
\usepackage{amsthm}
\usepackage{amsmath}
\usepackage{amssymb}
\usepackage{multirow}
\theoremstyle{plain}

\usepackage{babel}
\providecommand{\theoremname}{Theorem}
\usepackage{comment}
\DeclareMathOperator*{\minimize}{minimize}
\newcommand{\hi}{\mathbf{h}^{(i)}}
\newcommand{\xii}{\mathbf{x}^{(i)}}

\ifCLASSOPTIONcompsoc
  \usepackage[nocompress]{cite}
\else
  \usepackage{cite}
\fi
\ifCLASSINFOpdf
  \usepackage[pdftex]{graphicx}
\else
\fi
\usepackage{algorithmic}
\usepackage{algorithm}

\usepackage{color}
\newtheorem{theorem}{Theorem}
\newtheorem{proposition}{Proposition}
\newtheorem{definition}{Definition}

\begin{document}
%
\title{Discriminative Optimization: Theory and Applications to Computer Vision Problems}
%
%
%
%

\author{Jayakorn~Vongkulbhisal,
        Fernando~De~la~Torre,
        and~Jo\~{a}o~P.~Costeira 
\IEEEcompsocitemizethanks{\IEEEcompsocthanksitem J. Vongkulbhisal is with the ECE Department, Carnegie Mellon University, USA, and the ISR - IST, Universidade de Lisboa, Portugal. E-mail: jvongkul@andrew.cmu.edu.
\IEEEcompsocthanksitem F. De la Torre is with Facebook Inc., and the Robotics Institute, Carnegie Mellon University, USA. E-mail: ftorre@cmu.edu.
\IEEEcompsocthanksitem J. P. Costeira is with the ISR - IST, Universidade de Lisboa, Portugal. E-mail: jpc@isr.ist.utl.pt}
\thanks{Manuscript received April 19, 2005; revised August 26, 2015.}}

%
%

\markboth{Journal of \LaTeX\ Class Files,~Vol.~14, No.~8, August~2015}%
{Shell \MakeLowercase{\textit{et al.}}: Bare Demo of IEEEtran.cls for Computer Society Journals}
%



\IEEEtitleabstractindextext{%
\begin{abstract}

Many computer vision problems are formulated as the optimization of a cost function. This approach faces two main challenges: \textit{(i)} designing a cost function with a local optimum at an acceptable solution, and \textit{(ii)} developing an efficient numerical method to search for one (or multiple) of these local optima. While designing such functions is feasible in the noiseless case, the stability and location of local optima are mostly unknown under noise, occlusion, or missing data. In practice, this can result in undesirable local optima or not having a local optimum in the expected place. On the other hand, numerical optimization algorithms in high-dimensional spaces are typically local and often rely on expensive first or second order information to guide the search. To overcome these limitations, this paper proposes Discriminative Optimization (DO), a method that learns search directions from data without the need of a cost function. Specifically, DO explicitly learns a sequence of updates in the search space that leads to stationary points that correspond to desired solutions. We provide a formal analysis of DO and illustrate its benefits in the problem of 3D point cloud registration, camera pose estimation, and image denoising. We show that DO performed comparably or outperformed state-of-the-art algorithms in terms of accuracy, robustness to perturbations, and computational efficiency.
\end{abstract}

\begin{IEEEkeywords}
Optimization, gradient methods, iterative methods, computer vision, supervised learning.
\end{IEEEkeywords}}

\maketitle

\IEEEdisplaynontitleabstractindextext

%
\IEEEpeerreviewmaketitle

\IEEEraisesectionheading{\section{Introduction}}

Mathematical optimization plays an important role for solving many computer vision problems. For instance, optical flow, camera calibration, homography estimation, and structure from motion are computer vision problems solved as optimization. Formulating computer vision problems as optimization problems faces two main challenges: \textit{(i)} Designing a cost function that has a local optimum that corresponds to a suitable solution. \textit{(ii)} Selecting an efficient and accurate algorithm for searching the parameter space. Conventionally, these two steps have been treated independently, leading to different cost functions and search algorithms. However, in the presence of noise, missing data, or inaccuracies of the model, this conventional approach can lead to undesirable local optima or even not having an optimum in the \textit{correct} solution. 

Consider Fig.~\ref{fig:firstFig}a-top which illustrates a 2D alignment problem in a case of noiseless data. A good cost function for this problem should have a global optimum when the two shapes overlap.  Fig.~\ref{fig:firstFig}b-top illustrates the level sets of the cost function for the Iterative Closest Point (ICP) algorithm~\cite{ICP_PAMI92} in the case of complete and noiseless data. Observe that there is a well-defined optimum and that it coincides with the ground truth. Given a cost function, the next step is to find a suitable algorithm that, given an initial configuration (green square), finds a local optimum. For this particular initialization, the ICP algorithm will converge to the ground truth (red diamond in Fig.~\ref{fig:firstFig}b-top), and Fig.~\ref{fig:firstFig}d-top shows the convergence region for ICP in green. However, in realistic scenarios with the presence of perturbations in the data, there is no guarantee that there will be a good local optimum in the expected solution, while the number of local optima can be large.
and Fig.~\ref{fig:firstFig}b-bottom show the level set representation for the ICP cost function in the case of corrupted data. We can see that the shape of cost function has changed dramatically: there are more local optima, and they do not necessarily correspond to the ground truth (red diamond). In this case, the ICP algorithm with an initialization in the green square will converge to a wrong optimum. It is important to observe that the cost function is {\em only} designed to have an optimum at the correct solution in the ideal case, but little is known about the behavior of this cost function in the {\em surroundings} of the optimum and how it will change with noise.

\begin{figure*}[!t]
\centering
\includegraphics[width=7in]{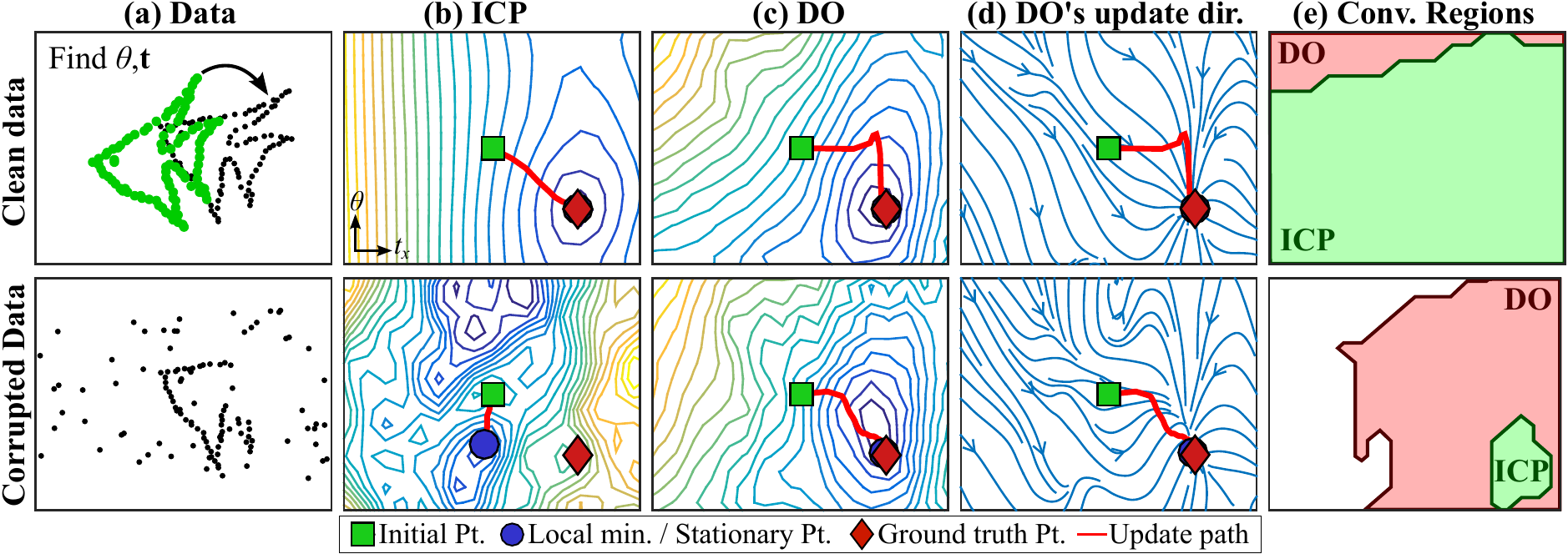}%
\caption[]{2D point alignment using ICP and DO. (a) Data. (b) Level sets of the cost function for ICP. We used the optimal matching at each parameter value to compute the $\ell_2$ cost. (c) Inferred level sets for the proposed DO. The level sets is approximately reconstructed using the surface reconstruction algorithm~\cite{Harker2008c} from the update directions of DO. (d) The update direction of DO.  (e) Regions of convergence for ICP and DO. See text for detailed description (best seen in color).}
\label{fig:firstFig}
\end{figure*}

To address the aforementioned problems, this paper proposes Discriminative Optimization (DO). DO exploits the fact that we often know from the training data where the solutions should be, whereas traditional approaches formulate optimization problems based on an ideal model. Rather than following a descent direction of a cost function, DO directly learns a sequence of update directions leading to a stationary point. These points are placed ``by design" in the desired solutions from training data. This approach has three main advantages.
First, since DO's directions are learned from training data, they take into account the perturbations in the \emph{neighborhood} of the ground truth, resulting in more robustness and a larger convergence region, as can be seen in Fig.~\ref{fig:firstFig}d. Second, because DO does not optimize any explicit function (e.g., $\ell_2$ registration error), it is less sensitive to model misfit and more robust to different types of perturbations. Fig.~\ref{fig:firstFig}c illustrates the contour level inferred from the update directions learned by DO. It can be seen that the curve levels have a local optimum on the ground truth and fewer local optima than ICP in Fig.~\ref{fig:firstFig}b. Fig.~\ref{fig:firstFig}e shows that the convergence regions of DO change little despite the perturbations, and always include the regions of ICP. Third, to compute update directions, traditional approaches require the cost function to be differentiable or continuous, whereas DO's directions can always be computed. We also provide a proof of DO's convergence in the training set. We named our approach DO to reflect the idea of learning to find a stationary point directly rather than that of optimizing a ``generative'' cost function.


In this work, we study the properties of DO and its relationship to mathematical programming. Based on this relationship, we propose a framework for designing features where the update directions can be interpreted as the gradient direction of an unknown cost function. We show that our approach can handle both ordered (e.g., image pixels) and unordered (e.g., set of feature matches) data types. We provide a synthetic experiment which confirms our interpretation. We apply DO to the problems of point cloud registration, camera pose estimation, and image denoising, and show that DO can obtain state-of-the-art results.

\section{Related Works}


\subsection{Optimization in computer vision}\label{sec:optInCV}
Many problems in computer vision involve solving inverse problems, that is to estimate a set of parameters $\mathbf{x}\in\mathbf{R}^p$ that satisfies $\mathbf{g}_j(\mathbf{x})=\mathbf{0}_d,j=1,\dots,J$, where $\mathbf{g}_j:\mathbf{R}^p\rightarrow\mathbf{R}^d$ models the phenomena of interest. For example, in camera resection problem~\cite{MultiviewGeom_2003}, $\mathbf{x}$ can represent the camera parameters and $\mathbf{g}_j$ the geometric projective error. Optimization-based framework tackles these problems by searching for the parameters $\mathbf{x}_*$ that optimize a certain cost function, which is typically designed to penalize the deviation of $\mathbf{g}_j$ from $\mathbf{0}_d$. Since selecting a cost function is a vital step for solving problems, there exists a large amount of research on the robustness properties of different penalty functions~\cite{Basri_VR1998,IRLS_COA2014,Barron_Arxiv17}. Instead of using a fixed cost function, many approaches use continuation methods to deform the cost function as the optimization is solved~\cite{Mobahi_EMMCVPR2015,Black_IJCV1996}. On the other hand, many computer vision problems are ill-posed~\cite{Bertero_IEEE1988}, and require some forms of regularization. This further leads to a combination of different penalty functions and requires setting the value of hyperparameters, which make the issue even more complicated.

Due to a large variety of design choices, it is not trivial to identify a suitable cost function for solving a problem. Instead of manually design a cost function, several works proposed to use machine learning techniques to learn a cost function from available training data. For example, kernel SVM~\cite{Avidan_PAMI2004}, boosting~\cite{Liu_PAMI2009}, metric learning~\cite{Nguyen_IJCV2010}, and nonlinear regressors~\cite{Paliouras2016} have been used to learn a cost function for image-based tracking, alignment, and pose estimation. Once a cost function is learned, the optimal parameters are solved using search algorithms such as descent methods or particle swarm optimization. However, a downside of these approaches is that they require the form of the cost function to be imposed, e.g.,~\cite{Nguyen_IJCV2010} requires to cost to be quadratic, thereby restricting the class of problems that they can solve. 
 
\subsection{Learning search direction}\label{sec:RW:learnDir}

Instead of using search directions from a cost function, recent works proposed to use learning techniques to directly compute such directions. This is done by learning a sequence of regressors that maps a feature vector to an update vector that points to the desired parameters. We refer to these algorithms as supervised sequential update (SSU). The concept of SSUs is similar to gradient boosting (GB)~\cite{Friedman_AnnStat2001,Mason_NIPS1999}, which uses weak learners to iteratively update parameter estimates. However, they differ in that GB performs update using a fixed feature vector, while SSUs also update the feature vector, which allows more information to be incorporated as the parameter is updated. Here, we provide a brief review of SSUs. 

Cascade pose regression~\cite{Dollar_CVPR2010} trains a sequence of random ferns for image-based object pose estimation. The paper also shows that the training error decreases exponentially under weak learner assumptions. \cite{Cao_CVPR2012} learns a sequence of boosted regressors that minimizes error in parameter space. Supervised descent method (SDM)~\cite{Xiong_CVPR2013,Xiong_SDM_Arxiv14} learns a sequence of linear maps as the averaged Jacobian matrices for minimizing nonlinear least-squares functions in the feature space. They also provided conditions for the error to strictly decrease in each iteration. More recent works include learning both Jacobian and Hessian matrices~\cite{Tzimiropoulos_CVPR2015}; running Gauss-Newton algorithm after SSU~\cite{Antonakos_ICIP2016}; using different maps in different regions of the parameter space~\cite{Xiong_CVPR2015}; and using recurrent neural network as the sequence of maps while also learns the feature~\cite{Trigeorgis_CVPR2016}. While the mentioned works learn the regressors in a sequential manner, Sequence of Learned Linear Predictor~\cite{Zimmermann_PAMI2009} first learns a set of linear maps then selects a subset to form a sequence of maps. 

We observe that most SSUs focus on image-based tracking and pose estimation. This is because the feature for these problems is rather obvious: they use intensity-based features such as intensity difference, SIFT, HOG, etc. Extending SSUs to other problems require designing new features. Our previous work~\cite{JV_CVPR2017} proposes a new feature and extends SSU to the problem of point cloud registration. Still, it is not straightforward to design features for other applications. In this work, we propose DO as a simple extension of previous SSUs. We study the its properties, propose a general framework for designing features, and apply DO to other computer vision problems, such as point cloud registration, camera pose estimation, and image denoising.

Recently, deep learning has received tremendous interest for its success in various tasks in computer vision and natural language processing~\cite{Goodfellow-et-al-2016}. Notably, some works use deep learning to solve optimization~\cite{andrychowicz2016learning,chen2016learning}, but they differ from DO since they still need a cost function to be defined. On the other hand, deep learning has been used for some applications similar to those in this paper, e.g., camera pose estimation~\cite{kendall2015posenet}, homography estimation~\cite{detone2016deep}, and image denoising~\cite{xie2012image}. However, DO differs from these works in that DO can combine the mathematical model of each problem with the available training data, while deep learning approaches are purely data-driven. The model-driven side of DO allows it to be interpreted and analyzed more easily, and significantly reduces the amount of data and computational power required for training.

 


\section{Discriminative Optimization (DO)}
In this section, we provide the motivation and describe the Discriminative Optimization (DO) framework.

\subsection{Motivation from fixed point iteration}
DO aims to learn a sequence of update maps (SUM) to update an initial parameter vector to a stationary point. The idea of DO is based on the fixed point iteration of the form
\begin{equation}
\mathbf{x}_{t+1}=\mathbf{x}_t-\Delta\mathbf{x}_t,\label{eq:fixedPointIter}
\end{equation}
where $\mathbf{x}_t\in\mathbb{R}^p$ is the parameter at step $t$, and $\Delta\mathbf{x}_t\in\mathbb{R}^p$ is the update vector. Eq.~\eqref{eq:fixedPointIter} is iterated until $\Delta\mathbf{x}_t$ vanishes, i.e., until a stationary point is reached. An example of fixed point iteration for solving optimization is the gradient descent algorithm~\cite{Boyd_convOpt}. Let $J:\mathbb{R}^p\rightarrow\mathbb{R}$ be a differentiable cost function. The gradient descent algorithm for minimizing $J$ is expressed as
\begin{equation}
\mathbf{x}_{t+1} = \mathbf{x}_t - \left.\mu_t\frac{\partial}{\partial \mathbf{x}}J(\mathbf{x})\right\vert_{\mathbf{x}=\mathbf{x}_t},
\label{eq:gradDescent}
\end{equation}
where $\mu_t$ is a step size. One can see that the scaled gradient is used as $\Delta\mathbf{x}_t$ in~\eqref{eq:fixedPointIter}, and it is known that the gradient vanishes when a stationary point is reached.

In contrast to gradient descent where the updates are derived from a cost function, DO learns the updates from the training data. The major advantages are that no cost function is explicitly designed and the neighborhoods around the solutions of the perturbed data are taken into account when the maps are learned.

\subsection{DO framework}

DO uses an update rule in the form of~\eqref{eq:fixedPointIter}. The update vector $\Delta\mathbf{x}_t$ is computed by mapping the output of a function $\mathbf{h}:\mathbb{R}^p\rightarrow\mathbb{R}^f$ with a sequence of matrices\footnote{Here, we used linear maps due to their simplicity and computational efficiency. However, other non-linear regression functions can be used in a straightforward manner.} $\mathbf{D}_t\in\mathbb{R}^{p \times f}$. Here, $\mathbf{h}$ is a function that encodes a representation of the data (e.g., $\mathbf{h}(\mathbf{x})$ extracts features from an image at position $\mathbf{x}$). Given an initial parameter $\mathbf{x}_0\in\mathbb{R}^p$, DO iteratively updates $\mathbf{x}_t,t=0,1,\dots$, using:
\begin{equation}
\mathbf{x}_{t+1} = \mathbf{x}_t - \mathbf{D}_{t+1}\mathbf{h}(\mathbf{x}_t),
\label{eq:updateRule}
\end{equation}
until convergence to a stationary point. 
The sequence of matrices $\mathbf{D}_{t+1},t=0,1,\dots$ learned from training data forms a sequence of update maps (SUM).

\subsubsection{Learning a SUM}
Suppose we are given a training set as a set of triplets $\{(\mathbf{x}^{(i)}_0,\mathbf{x}_*^{(i)},\mathbf{h}^{(i)})\}_{i=1}^N$, where $\mathbf{x}^{(i)}_0\in\mathbb{R}^p$ is the initial parameter for the $i^{th}$ problem instance (e.g., the $i^{th}$ image),  $\mathbf{x}_*^{(i)}\in\mathbb{R}^p$ is the ground truth parameter (e.g., position of the object on the image), and $\mathbf{h}^{(i)}:\mathbb{R}^p\rightarrow\mathbb{R}^f$ extract features from the $i^{th}$ problem instance. The goal of DO is to learn a sequence of update maps $\{\mathbf{\mathbf{D}}_t\}_t$ that updates $\mathbf{x}^{(i)}_0$ to $\mathbf{x}_*^{(i)}$. To learn the maps, we minimize the least-square error:
\begin{equation}
\mathbf{D}_{t+1}=\arg\min_{\tilde{\mathbf{D}}} \frac{1}{N}\sum_{i=1}^N \Vert\mathbf{x}_*^{(i)} - \mathbf{x}^{(i)}_{t} + \tilde{\mathbf{D}}\mathbf{h}^{(i)}(\mathbf{x}^{(i)}_t)\Vert_2^2,
\label{eq:learning}
\end{equation}
where $\Vert\cdot\Vert_2$ is the $\ell_2$ norm. After we learn a map $\mathbf{D}_{t+1}$, we update each $\mathbf{x}^{(i)}_{t}$ using~\eqref{eq:updateRule}, then proceed to learn the next map. This process is repeated until some terminating conditions, such as until the error does not decrease much or a maximum number of iterations is reached. To see why~\eqref{eq:learning} learns stationary points, we can see that for $i$ with $\mathbf{x}^{(i)}_t\approx\mathbf{x}^{(i)}_*$,~\eqref{eq:learning} will force $\mathbf{D}_{t+1}\mathbf{h}^{(i)}(\mathbf{x}^{(i)}_t)$ to be close to zero, thereby inducing a stationary point around $\mathbf{x}^{(i)}_*$. In practice, we use ridge regression to learn the maps to prevent overfitting:
\begin{equation}
\minimize_{\tilde{\mathbf{D}}} \frac{1}{N}\sum_{i=1}^N \Vert\mathbf{x}_*^{(i)} - \mathbf{x}^{(i)}_t + \tilde{\mathbf{D}}\mathbf{h}^{(i)}(\mathbf{x}^{(i)}_t)\Vert_2^2+
\lambda\Vert\tilde{\mathbf{D}}\Vert_F^2,
\label{eq:learningReg}
\end{equation}
where $\Vert\cdot\Vert_F$ is the Frobenius norm, and $\lambda$ is a hyperparameter. The pseudocode for training a SUM is shown in Alg.~\ref{alg:LearnSUM}.

\subsubsection{Solving a new problem instance}
To solve a new problem instance with an unseen function $\mathbf{h}$ and an initialization $\mathbf{x}_0$, we update $\mathbf{x}_t,t=0,1,\dots$ with the obtained SUM using~\eqref{eq:updateRule} until a stationary point is reached. However, in practice, the number of maps is finite, say $T$ maps. We observed in many cases that the update at the $T^{th}$ iteration is still large, which means the stationary point is still not reached, and that $\mathbf{x}_T$ is far from the true solution. For example, in the registration task, the rotation between initial orientation and the solution might be so large that we cannot obtain the solution within a fixed number of update iterations.
To overcome this problem, we keep updating $\mathbf{x}$ using the $T^{th}$ map until the update is small or the maximum number of iterations is reached. This approach makes DO different from previous works in Sec.~\ref{sec:RW:learnDir}, where the updates are only performed up to the number of maps. Alg.~\ref{alg:updateParam} shows the pseudocode for updating the parameters.


\begin{algorithm}[t]
\caption{Training a sequence of update maps (SUM)}
\label{alg:LearnSUM}
\begin{algorithmic}[1]
\REQUIRE $\{(\mathbf{x}^{(i)}_0,\mathbf{x}_*^{(i)},\mathbf{h}^{(i)})\}_{i=1}^N,T,\lambda$
\ENSURE $\{\mathbf{D}_t\}_{t=1}^T$
\FOR{$t = 0$ \TO $T-1$}
\STATE Compute $\mathbf{D}_{t+1}$ with~\eqref{eq:learningReg}.
\FOR{$i = 1$ \TO $N$}
\STATE Update  $\mathbf{x}_{t+1}^{(i)}:=\mathbf{x}_{t}^{(i)}-\mathbf{D}_{t+1}\mathbf{h}^{(i)}(\mathbf{x}_{t}^{(i)})$.
\ENDFOR
\ENDFOR
\end{algorithmic}
\end{algorithm}

\begin{algorithm}[t]
\caption{Searching for a stationary point}
\label{alg:updateParam}
\begin{algorithmic}[1]
\REQUIRE $\mathbf{x}_0,\mathbf{h},\{\mathbf{D}_t\}_{t=1}^T,maxIter,\epsilon$
\ENSURE $\mathbf{x}$
\STATE Set $\mathbf{x}:=\mathbf{x}_0$
\FOR{$t=1$ \TO $T$}
\STATE Update $\mathbf{x} := \mathbf{x}-\mathbf{D}_t\mathbf{h}(\mathbf{x})$
\ENDFOR
\STATE Set $iter := T+1$.
\WHILE{$\Vert\mathbf{D}_T\mathbf{h}(\mathbf{x})\Vert\geq\epsilon$ \AND $iter \leq maxIter$}
\STATE Update $\mathbf{x} := \mathbf{x}-\mathbf{D}_T\mathbf{h}(\mathbf{x})$
\STATE Update $iter := iter + 1$
\ENDWHILE
\end{algorithmic}
\end{algorithm}

\section{Theoretical Analysis of DO}\label{sec:analysis}

In this section, we analyze the theoretical properties of DO. Specifically, we discuss the conditions for the convergence of the training error, and the relation between DO and mathematical optimization.

\subsection{Convergence of training error}
Here, we show that under a weak assumption on $\mathbf{h}^{(i)}$, we can learn a SUM that updates $\mathbf{x}^{(i)}_0$ to $\mathbf{x}_*^{(i)}$, i.e., the training error converges to zero. First, we define the \emph{monotonicity at a point} condition:

\begin{definition}
(Monotonicity at a point)
A function $\mathbf{f}:\mathbb{R}^p\rightarrow\mathbb{R}^p$ is

(i) \emph{monotone at $\mathbf{x}_*\in\mathbb{R}^p$} if
\begin{equation}
(\mathbf{x}-\mathbf{x}_*)^\top\mathbf{f}(\mathbf{x})\geq 0
\end{equation} 
for all $\mathbf{x}\in\mathbb{R}^p$, 

(ii) \emph{strictly monotone at $\mathbf{x}_*\in\mathbb{R}^p$} if
\begin{equation}
(\mathbf{x}-\mathbf{x}_*)^\top\mathbf{f}(\mathbf{x})\geq 0
\end{equation} 
for all $\mathbf{x}\in\mathbb{R}^p$ and the equality holds only at $\mathbf{x}=\mathbf{x}_*$,

(iii) \emph{strongly monotone at $\mathbf{x}_*\in\mathbb{R}^p$} if
\begin{equation}
(\mathbf{x}-\mathbf{x}_*)^\top\mathbf{f}(\mathbf{x})\geq m\Vert\mathbf{x}-\mathbf{x}_*\Vert_2^2
\end{equation}
for some $m>0$ and all $\mathbf{x}\in\mathbb{R}^p$.
\end{definition}

It can be seen that if $\mathbf{f}$ is strongly monotone at $\mathbf{x}$ then $\mathbf{f}$ is strictly monotone at $\mathbf{x}$, and if $\mathbf{f}$ is strictly monotone at $\mathbf{x}$ then $\mathbf{f}$ is monotone at $\mathbf{x}$. With the above definition, we obtain the following result:

\begin{theorem}
{\bf (Convergence of SUM's training error)} Given a training set $\{(\mathbf{x}^{(i)}_0,\mathbf{x}_*^{(i)},\mathbf{h}^{(i)})\}_{i=1}^N$, if there exists a linear map $\mathbf{\hat{D}}\in\mathbb{R}^{p\times f}$
where $\hat{\mathbf{D}}\mathbf{h}^{(i)}$ is strictly monotone at $\mathbf{x}_*^{(i)}$ for all $i$, and if there exists an $i$ where $\mathbf{x}^{(i)}_t\neq\mathbf{x}_*^{(i)}$, then the update rule:
\begin{equation}
\mathbf{x}^{(i)}_{t+1}=\mathbf{x}^{(i)}_{t}-
\mathbf{D}_{t+1}\mathbf{h}^{(i)}(\mathbf{x}^{(i)}_{t}),
\end{equation}
with $\mathbf{D}_{t+1}\subset\mathbb{R}^{p\times f}$
obtained from~\eqref{eq:learning}, guarantees that the training error strictly decreases in each iteration:
\begin{equation}
\sum_{i=1}^{N}
\Vert\mathbf{x}_*^{(i)}-\mathbf{x}^{(i)}_{t+1}\Vert_2^{2}
<
\sum_{i=1}^{N}
\Vert\mathbf{x}_*^{(i)}-\mathbf{x}^{(i)}_t\Vert_2^{2}.
\end{equation}
Moreover, if $\hat{\mathbf{D}}\mathbf{h}^{(i)}$ is strongly monotone at $\mathbf{x}_*^{(i)}$, and if there exist $M>0,H\geq0$ such that $\Vert \hat{\mathbf{D}}\mathbf{h}^{(i)}(\mathbf{x}^{(i)})\Vert_2^2\leq H+M\Vert\mathbf{x}_*^{(i)}-\mathbf{x}^{(i)}\Vert_2^2$ for all $i$, then the training error converges to zero. If $H=0$ then the error converges to zero linearly.
\label{thm:trainErr}
\end{theorem}
The proof of Thm.~\ref{thm:trainErr} is provided in the appendix. In words, Thm.~\ref{thm:trainErr} says that if each instance $i$ is similar in the sense that each $\hat{\mathbf{D}}\mathbf{h}^{(i)}$ is strictly monotone at $\mathbf{x}_*^{(i)}$, then sequentially learning the optimal maps with~\eqref{eq:learning} guarantees that the training error strictly reduces in each iteration. If $\hat{\mathbf{D}}\mathbf{h}^{(i)}$ is strongly monotone at $\mathbf{x}_*^{(i)}$ and upperbounded then the error converges to zero. Note that $\mathbf{h}^{(i)}$ is not required to be differentiable or continuous. Xiong and De la Torre~\cite{Xiong_SDM_Arxiv14} also presents a convergence result for a similar update rule, but it shows the strict reduction of error of a \emph{single} function under a \emph{single ideal} map. It also requires an additional condition called `Lipschitz at a point,' This condition is necessary for bounding the norm of the map, otherwise the update can be too large, preventing the reduction in error. In contrast, Thm.~\ref{thm:trainErr} explains the convergence of \emph{multiple} functions under the same SUM learned from the data, where the each learned map $\mathbf{D}_t$ can be different from the ideal map $\hat{\mathbf{D}}$. To ensure reduction of error, Thm.~\ref{thm:trainErr} also does not require the `Lipschitz at a point' the norms of the maps are adjusted based on the training data. Meanwhile, to ensure convergence to zero, Thm.~\ref{thm:trainErr} requires an upperbound which can be thought of as a relaxed version of `Lipschitz at a point' (note that $\hat{\mathbf{D}}\mathbf{h}^{(i)}(\mathbf{x}_*^{(i)})$ does not need to be $\mathbf{0}_p$). These weaker assumptions have an important implication as it allows robust discontinuous features, such as HOG in~\cite{Xiong_SDM_Arxiv14}, to be used as $\mathbf{h}^{(i)}$. Finally, we wish to point out that Thm.~\ref{thm:trainErr} guarantees the reduction in the average error, not the error of each instance $i$.

\subsection{Relation to mathematical programming}
In this section, we explore the relation between DO and mathematical programming. Specifically, we show that monotonicity-at-a-point is a generalization of monotonicity and pseudomonotonocity, which are the properties of the gradient of convex and pseudoconvex functions~\cite{Rockafellar_cvxAna1970,Karamardian_JOTA1990}. Understanding this relation leads to a framework for designing $\mathbf{h}$ in Sec.~\ref{sec:designH}. We begin this section by providing definitions and propositions relating generalized convexity and monotonicity, then we provide our result in the end.

A pseudoconvex function is defined as follows.

\begin{definition}
(Pseudoconvexity~\cite{Karamardian_JOTA1990}) A differentiable function $f:\mathbb{R}^p\rightarrow\mathbb{R}$ is 

(i) \emph{pseudoconvex} if for any distinct points $\mathbf{x},\mathbf{x}'\in\mathbb{R}^p$, 
\begin{equation}
(\mathbf{x}-\mathbf{x}')^\top \nabla f(\mathbf{x}') \geq 0 \implies f(\mathbf{x})\geq f(\mathbf{x}'),
\end{equation}

(ii) \emph{strictly pseudoconvex} if for any distinct points $\mathbf{x},\mathbf{x}'\in\mathbb{R}^p$, 
\begin{equation}
(\mathbf{x}-\mathbf{x}')^\top \nabla f(\mathbf{x}') \geq 0 \implies f(\mathbf{x}) > f(\mathbf{x}'),
\end{equation}

(iii) \emph{strongly pseudoconvex} if there exists $m>0$ such that for any distinct points $\mathbf{x},\mathbf{x}'\in\mathbb{R}^p$,
\begin{equation}
(\mathbf{x}-\mathbf{x}')^\top \nabla f(\mathbf{x}') \geq 0 \implies f(\mathbf{x})\geq f(\mathbf{x}')+m\Vert\mathbf{x}-\mathbf{x}'\Vert_2^2.
\end{equation}
\end{definition} 

Fig.~\ref{fig:meanMedian}d shows examples of pseudoconvex functions. In essence, pseudoconvex functions are differentiable functions where the sublevel sets are convex and all stationary points are global minima. Pseudoconvex functions generalize convex functions: all differentiable convex functions are pseudoconvex. Pseudoconvex functions are used as penalty functions for their stronger robustness than convex ones~\cite{Barron_Arxiv17,Black_IJCV1996,Ochs_CVPR2013}. Next, we introduce pseudomonotonicity. 

\begin{definition}
(Pseudomonotonicity~\cite{Karamardian_JOTA1990}) A function $\mathbf{f}:\mathbb{R}^p\rightarrow\mathbb{R}^p$ is 

(i) \emph{pseudomonotone} if for any distinct points $\mathbf{x},\mathbf{x}'\in\mathbb{R}^p$,
\begin{equation}
(\mathbf{x}-\mathbf{x}')^\top \mathbf{f}(\mathbf{x}') \geq 0 \implies (\mathbf{x}-\mathbf{x}')^\top \mathbf{f}(\mathbf{x}) \geq 0,
\end{equation}

(ii) \emph{strictly pseudomonotone} if for any distinct points $\mathbf{x},\mathbf{x}'\in\mathbb{R}^p$,
\begin{equation}
(\mathbf{x}-\mathbf{x}')^\top \mathbf{f}(\mathbf{x}') \geq 0 \implies (\mathbf{x}-\mathbf{x}')^\top \mathbf{f}(\mathbf{x}) > 0,
\end{equation}

(iii) \emph{strongly pseudomonotone} if there exists $m>0$ such that for any distinct points $\mathbf{x},\mathbf{x}'\in\mathbb{R}^p$,
\begin{equation}
(\mathbf{x}-\mathbf{x}')^\top \mathbf{f}(\mathbf{x}') \geq 0 \implies (\mathbf{x}-\mathbf{x}')^\top \mathbf{f}(\mathbf{x}) \geq m\Vert\mathbf{x}-\mathbf{x}'\Vert_2^2.
\end{equation}
\end{definition}

It can also be shown that monotone (resp., strictly, strongly) functions are pseudomonotone (resp., strictly, strongly)~\cite{Karamardian_JOTA1990}. The following propositions provides a relation between the gradients of pseudoconvex functions and pseudomonotonicity.

\begin{proposition}\label{prop:gradCvx}
(Convexity and monotonicity~\cite{Karamardian_JOTA1990}) A differentiable function $f:\mathbb{R}^p\rightarrow\mathbb{R}$ is pseudoconvex (resp., strictly, strongly) if and only if its gradient is pseudomonotone (resp., strictly, strongly).
\end{proposition}

Next, we provide our result on the relation between monotonicity-at-a-point and pseudomonotonicity.

\begin{proposition}
(Pseudomonotonicity and monotonicity at a point) If a function $\mathbf{f}:\mathbb{R}^p\rightarrow \mathbb{R}^p$ is pseudomonotone (resp., strictly, strongly) and $\mathbf{f}(\mathbf{x}_*)=\mathbf{0}_p$, then $\mathbf{f}$ is monotone (resp., strictly, strongly) at $\mathbf{x}_*$. 
\label{prop:psMonIsMonAtPt}
\end{proposition} 
The converse of the proposition is not true. For example, 
$\mathbf{f}(\mathbf{x})=[x_1x_2^2+x_1,x_2x_1^2+x_2]{^\top}$
is strictly monotone at $\mathbf{0}_2$, but not strictly pseudomonotone (counterexample at $\mathbf{x}=(1,2)$ and $\mathbf{y}=(2,1)$). Prop.~\ref{prop:psMonIsMonAtPt} shows that monotonicity-at-a-point is a generalization of pseudomonotonicity, implying that the conditions in Thm.~\ref{thm:trainErr} are weaker than the conditions for the gradient maps of pseudoconvex and convex functions.  

\section{Designing $\mathbf{h}$}\label{sec:designH}
The function $\mathbf{h}$ which provides information about each problem instance is crucial for solving a problem. In this section, we describe a framework to design $\mathbf{h}$ for solving a class of problem based on our analysis in Sec.~\ref{sec:analysis}. We are motivated by the observation that many problems in computer vision aim to find $\mathbf{x}$ such that $\mathbf{g}_j(\mathbf{x})=\mathbf{0}_d,j=1,\dots,J$, where $\mathbf{g}_j:\mathbb{R}^p\rightarrow\mathbb{R}^d$ models the problem of interest (see Sec.~\ref{sec:optInCV}). To solve such problem, one may formulate an optimization problem of the form 
\begin{equation}
\minimize_\mathbf{x}\Phi(\mathbf{x})=\frac{1}{J}\sum_{j=1}^J\varphi(\mathbf{g}_j(\mathbf{x})),
\label{eq:genericCost}
\end{equation}
where $\varphi:\mathbb{R}^d\rightarrow\mathbb{R}$ is a penalty function, e.g., sum of squares, $\ell_1$ norm, etc. If $\Phi({x})$ is differentiable, then we can use gradient descent to find a minimum and returns it as the solution. The choice of $\varphi$ has a strong impact on the solution in terms of robustness to different perturbations, and it is not straightforward to select $\varphi$ that will account for perturbations in real data. The following framework is based on the concept of using training data to learn the update directions that mimic gradient descent of an unknown $\varphi$, thereby bypassing the manual selection of $\varphi$.

\subsection{$\mathbf{h}$ from the gradient of an unknown penalty function}
For simplicity, we assume $\Phi(\mathbf{x})$ is differentiable, but the following approach also applies when it does not. Let us observe its derivative:
\begin{equation}
\frac{\partial\Phi(\mathbf{x})}{\partial\mathbf{x}}
=
\frac{1}{J}\frac{\partial}{\partial\mathbf{x}}\sum_{j=1}^J\varphi(\mathbf{g}_j)
=
\frac{1}{J}\sum_{j=1}^J\left[\frac{\partial \mathbf{g}_j}{\partial \mathbf{x}}\right]^\top
\frac{\partial \varphi(\mathbf{g}_j)}{\partial \mathbf{g}_j},
\label{eq:divError}
\end{equation}
where we express $\mathbf{g}_j(\mathbf{x})$ as $\mathbf{g}_j$ to reduce notation clutter. We can see that the form of $\varphi$ affects only the last term in the RHS of~\eqref{eq:divError}, while the Jacobian $\frac{\partial\mathbf{g}_j}{\partial\mathbf{x}}$ does not depend on it. Since different $\varphi$'s are robust to different perturbations, this last term determines the robustness of the solution. Here, we will use DO to learn this term from a set of training data.
 
In order to do so, we need to express~\eqref{eq:divError} as $\mathbf{D}\mathbf{h}$. First, we rewrite~\eqref{eq:divError} as the update vector $\Delta \mathbf{x}$, where we replace the derivative of $\varphi$ with a generic function $\mathbf{\phi}:\mathbf{R}^d\rightarrow\mathbf{R}^d$:
\begin{align}
\Delta \mathbf{x} & =
\frac{1}{J}\sum_{j=1}^J\left[\frac{\partial \mathbf{g}_j}{\partial \mathbf{x}}\right]^\top
\mathbf{\phi}(\mathbf{g}_j)\\
& = \frac{1}{J}\sum_{j=1}^J\sum_{k=1}^d\left[\frac{\partial \mathbf{g}_j}{\partial \mathbf{x}}\right]^\top_{k,:}
[\mathbf{\phi}(\mathbf{g}_j)]_k,\label{eq:updateGradIndex}
\end{align}
where $[\mathbf{Y}]_{k,:}$ is row $k$ of $\mathbf{Y}$, and $[\mathbf{y}]_k$ is element $k$ of $\mathbf{y}$. We then rewrite~\eqref{eq:updateGradIndex} as the following convolution:
\begin{equation}
\small{ \Delta \mathbf{x} = 
\frac{1}{J}\sum_{j=1}^J\sum_{k=1}^d
\left[\frac{\partial \mathbf{g}_j}{\partial \mathbf{x}}\right]^\top_{k,:}
\int_{\mathbb{R}^d} [\mathbf{\phi}(\mathbf{v})]_k \delta(\mathbf{v}-\mathbf{g}_j) d\mathbf{v}}, \label{eq:convKdeltaDv}
\end{equation}
where $\delta(\mathbf{v})$ is the Dirac delta function.
It can be seen that~\eqref{eq:convKdeltaDv} is equivalent to~\eqref{eq:divError}, while being linear in $\mathbf{\phi}$. This allows us to learn $\mathbf{\phi}$ using linear least squares. To do so, we will express~\eqref{eq:convKdeltaDv} in the form of $\mathbf{D}\mathbf{h}$. For simplicity, we will look at the element $l$ of $\Delta \mathbf{x}$:
{\small
\begin{align}
[\Delta \mathbf{x}]_l & = 
\frac{1}{J}\sum_{j=1}^J\sum_{k=1}^d
\left[\frac{\partial \mathbf{g}_j}{\partial \mathbf{x}}\right]_{k,l}
\int_{\mathbb{R}^d} [\mathbf{\phi}(\mathbf{v})]_k \delta(\mathbf{v}-\mathbf{g}_j) d\mathbf{v},
\\
&=
\frac{1}{J}\sum_{k=1}^d\int_{\mathbb{R}^d} 
[\mathbf{\phi}(\mathbf{v})]_k
\left(\sum_{j=1}^J\left[\frac{\partial \mathbf{g}_j}{\partial \mathbf{x}}\right]_{k,l}
\delta(\mathbf{v}-\mathbf{g}_j)\right) d\mathbf{v},
\\
&=
\sum_{k=1}^d\int_{\mathbb{R}^d}\mathbf{D}(\mathbf{v},k)\mathbf{h}(\mathbf{v},k,l;\mathbf{x})d\mathbf{v}\label{eq:expDh}.
\end{align}
}
Eq.~\ref{eq:expDh} expresses $[\Delta\mathbf{x}]_l$ as an inner product between $\mathbf{D}$ and $\mathbf{h}$ over $\mathbf{v}$ and $k$, where 
\begin{align}
\mathbf{D}(\mathbf{v},k) & = [\mathbf{\phi}(\mathbf{v})]_k,\label{eq:designH-D}
\\
\mathbf{h}(\mathbf{v},k,l;\mathbf{x}) & = \frac{1}{J}\sum_{j=1}^J
\left[\frac{\partial \mathbf{g}_j}{\partial \mathbf{x}}\right]_{k,l}
\delta(\mathbf{v}-\mathbf{g}_j).\label{eq:designH-h}
\end{align}
The following results discusses the convergence of training data when $\mathbf{h}$ in~\eqref{eq:designH-h} is used.

\begin{proposition}
(Convergence of the training error with an unknown penalty function) Given a training set $\{(\mathbf{x}_0^{(i)},\mathbf{x}_*^{(i)},\{\mathbf{g}^{(i)}_j\}_{j=1}^{J_i})\}_{i=1}^N$, where $\mathbf{x}_0^{(i)},\mathbf{x}_*^{(i)}\in\mathbb{R}^{p}$ and $\mathbf{g}_j^{(i)}:\mathbb{R}^{p}\rightarrow\mathbb{R}^d$ differentiable, if there exists a function $\varphi:\mathbb{R}^d\rightarrow\mathbb{R}$ such that for each $i$, $\sum_{j=1}^{J_i} \varphi(\mathbf{g}^{(i)}_j(\mathbf{x}^{(i)}))$ is differentiable strictly pseudoconvex with the minimum at $\mathbf{x}_*^{(i)}$, then the training error of DO with $\mathbf{h}$ from~\eqref{eq:designH-h} strictly decreases in each iteration. Alternatively, if $\sum_{j=1}^{J_i} \varphi(\mathbf{g}^{(i)}_j(\mathbf{x}^{(i)}))$ is differentiable strongly pseudoconvex with Lipschitz continuous gradient, then the training error of DO converges to zero.\label{prop:DesignH}
\end{proposition}

Under similar conditions, we can also show the same convergence results for $\sum_{j=1}^{J_i} \varphi(\mathbf{g}^{(i)}_j(\mathbf{x}^{(i)}))$ that is nondifferentiable strictly and strongly convex functions. Roughly speaking, Prop.~\ref{prop:DesignH} says that if there exists a penalty function $\varphi$ such that for each $i$ the global minimum of~\eqref{eq:genericCost} is at $\mathbf{x}_*^{(i)}$ with no other local minima, then using~\eqref{eq:designH-h} allows us to learn $\{\mathbf{D}_t\}$ for DO. Note that we do not need to explicitly know what such penalty function is. Thus, we can say that using~\eqref{eq:designH-h} is equivalent to learning a surrogate of the gradient of an unknown cost function. This illustrates the potential of DO as a tool for solving a broad class of problems where the penalty function $\varphi$ is unknown.

\subsection{Computing $\mathbf{h}$}


Eq.~\eqref{eq:designH-h} expresses $\mathbf{h}$ as a function. To compute $\mathbf{h}$ in practice, we need to express it as a vector. To do so, we will convert $\mathbf{h}$ in~\eqref{eq:designH-h} into a discrete grid, then vectorize it. Specifically, we first discretize $\delta(\mathbf{v}-\mathbf{g}_j)$ into a $d$-dimensional grid with $r$ bins in each dimension, where a bin evaluates to $1$ if $\mathbf{g}_j$ is discretized to that bin, and $0$ for all other bins. Let us denote this grid as $\bar{\delta}_j$, and let $\gamma:\mathbb{R}\rightarrow\{1,\dots,r\}$ be a function where $\gamma(y)$ returns the index that $y$ discretizes to. We can express the vectorized $\bar{\delta}_j$ as the following Kronecker product of standard bases:
\begin{equation}
vec(\bar{\mathbf{\delta}}_j) = \mathbf{e}_{\gamma([\mathbf{g}_j]_1)}\otimes\dots\otimes\mathbf{e}_{\gamma([\mathbf{g}_j]_d)}
= \bigotimes_{\alpha=1}^d\mathbf{e}_{\gamma([\mathbf{g}_j]_{\alpha})}\in\{0,1\}
^{r^d}.
\end{equation}
With this discretization, we can express $\mathbf{h}$ in~\eqref{eq:designH-h} in a discrete form as
\begin{equation}
\mathbf{h}(k,l;\mathbf{x}) = \frac{1}{J}\sum_{j=1}^J\left[\frac{\partial \mathbf{g}_j}{\partial\mathbf{x}}\right]_{k,l}
\bigotimes_{\alpha=1}^d\mathbf{e}_{\gamma([\mathbf{g}_j]_{\alpha})}.
\end{equation}
By concatenating $\mathbf{h}(k,l;\mathbf{x})$ over $k$ and $l$, we obtain the final form of $\mathbf{h}$ as 
\begin{equation}
\mathbf{h}(\mathbf{x}) = \frac{1}{J}\sum_{j=1}^J\bigoplus_{l=1}^p \bigoplus_{k=1}^d
\left[\frac{\partial \mathbf{g}_j}{\partial\mathbf{x}}\right]_{k,l}
\bigotimes_{\alpha=1}^d\mathbf{e}_{\gamma([\mathbf{g}_j]_{\alpha})},\label{eq:hFinalForm}
\end{equation}
where $\bigoplus$ denotes vector concatenation. The dimension of $\mathbf{h}$ is $pdr^d$. We show how to apply~\eqref{eq:hFinalForm} to applications in Sec.~\ref{sec:Experiments}. Note that the above approach is one way of designing $\mathbf{h}$ to use with SUM. It is possible to use different form of $\mathbf{h}$ (e.g., see Sec.~\ref{exp:3Dregis}), or replace $\mathbf{D}$ with a nonlinear map.

\section{Experiments}\label{sec:Experiments}
In this section, we first provide an intuition into DO with an analytical example, then we apply DO to three computer vision problems: 3D point cloud registration, camera pose estimation, and image denoising. All experiments were performed in MATLAB on a single thread on an Intel i7-4790 3.60GHz computer with 16GB memory.

\subsection{Optimization with unknown 1D cost functions}
In this experiment, we demonstrate DO's potential in solving 1D problems without an explicit cost function. Specifically, given a set of number $X=\{x_1,x_2,\dots,x_J\}$, we are interested in finding the solution $\hat{x}$ of the problem 
\begin{equation}
g_j(\hat{x})=0=\hat{x}-x_j,j=1,\dots,J.
\end{equation}
A typical approach to solve this problem is to solve the optimization
\begin{equation}
P:\minimize_{\hat{x}:\hat{x}=x_j+\epsilon_j}\sum_{j=1}^J\varphi\left(\epsilon_j\right)\equiv\minimize_{\hat{x}}\sum_{j=1}^n\varphi\left(\hat{x}-x_j\right),\label{eq:findMean}
\end{equation}
for some function $\varphi$. The form of $\varphi$ depends on the assumption on the distribution of $\epsilon_i$, e.g., the maximum likelihood estimation for i.i.d.~Gaussian $\epsilon_j$ would use $\varphi(x)=x^2$. If the an explicit form of $\varphi$ is known, then one can compute $\hat{x}_*$ in closed form (e.g., $\varphi$ is squared value or absolute value) or with an iterative algorithm. However, using a $\varphi$ that mismatches with the underlying distribution of $\epsilon_j$ could lead to an optimal, but incorrect, solution $\hat{x}_*$. Here, we will use DO to solve for $\hat{x}_*$ from a set of training data.

For this problem, we defined 6 $\varphi_\beta$'s as follows:
\begin{align}
\varphi_1(x)=& |x|,\label{eq:distCvxComb}\\
\varphi_2(x)=& 0.35|x|^{4.32}+0.15|x|^{1.23},\label{eq:distCvxComb}\\
\varphi_3(x)=& (3+sgn(x))x^2/4,\label{eq:distCvxComb}\\
\varphi_4(x)=& |x|^{0.7},\label{eq:distCvxComb}\\
\varphi_5(x)=& 1-exp(-2x^2),\label{eq:distCvxComb}\\
\varphi_6(x)=& 1-exp(-8x^2).\label{eq:distCvxComb}
\end{align}
The first 3 $\varphi_i$'s are convex, where $\varphi_1$ is a nonsmooth function; $\varphi_2$ is a combination of different powers; $\varphi_3$ is an asymmetric function (i.e., $\varphi_3(x)\neq \varphi_3(-x)$). The latter 3 $\varphi_i$'s are pseudoconvex, where $\varphi_4$ has exponents smaller than $1$; while $\varphi_5$ and $\varphi_6$ are inverted Gaussian function with different widths. Pseudoconvex functions are typically used as robust penalty functions\cite{Ochs_CVPR2013} because they penalize outliers less than convex functions. Recall that sum of pseudoconvex functions may not be pseudoconvex, and can have multiple local minima. The graphs of the functions and the gradient\footnote{Here, we abuse the word \textit{gradient} to include subdifferential for nonsmooth convex functions and generalized subdifferential for nonconvex functions~\cite{Hadjisavvas_EncOpt2001}.} are shown in Fig.~\ref{fig:meanMedian}a,b,d,e. We call the problem in~\eqref{eq:findMean} that uses $\varphi=\varphi_\beta$ as $P_\beta$.

\begin{figure}[!t]
\centering
\includegraphics[width=3.45in]{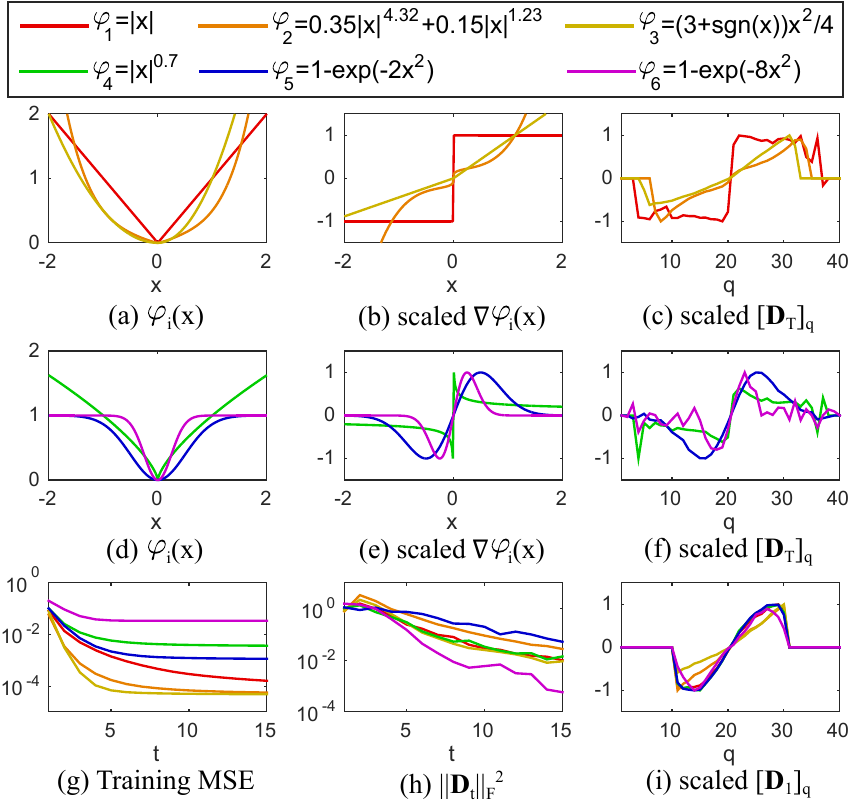}%
\caption{Learning to solve unknown cost functions. (a-c) show three convex functions, their gradients, and the learned $\mathbf{D}_T$ for each function. (d-f) show similar figures for pseudoconvex functions. (g) shows training error in each step $t$. (h) shows the squared norm of the maps $\mathbf{D}_{t}$. (i) shows the first map of each function.}
\label{fig:meanMedian}
\end{figure}

We generate the training data for $P_\beta$ as $X_\beta=\{(X^{(i)}_\beta,\hat{x}^{(i)}_{0,\beta},\hat{x}_{*,\beta}^{(i)})\}_{i=1}^{10000}$ where $X^{(i)}_\beta=\{x_{1,\beta}^{(i)},\dots,x_{J_i,\beta}^{(i)}\}\subset[-1,1]$; $\hat{x}^{(i)}_{0,\beta}=0$ is the initial estimate; and $\hat{x}_{*,\beta}^{(i)}$ is the global minimizer of $P_\beta$ with the data $X^{(i)}_\beta$. To find the minimizers, we use \texttt{fminunc} for convex functions, and grid search with the step size of $0.0001$ for nonconvex functions. We trained the SUMs using the $\mathbf{h}$ in Sec.~\ref{sec:designH}, which in this case is simply
\begin{equation}
\mathbf{h}(\hat{x}) = \frac{1}{J}\sum_{j=1}^J\mathbf{e}_{\gamma(\hat{x}-x_j)}.
\end{equation}
We use $[-2,2]$ as the range of $\hat{x}-x_j$, and discretize it into $r=40$ bins. Let us denote the maps that learn from $X_\beta$ as ${SUM}_\beta$. To illustrate the training error, we train up to 15 maps for each $\beta$, but for test we set the number of maps $T$ to the last map that reduce the training RMSE more than $0.005$. During test, we set $\epsilon=10^{-3}$ and $maxIter=100$.

Fig.~\ref{fig:meanMedian}c,f show the scaled maps $\mathbf{D}_T$ for each $\beta$. We can see that the maps resemble the gradient of their respective functions, suggesting that DO can learn the gradients from training data without explicit access to the cost functions. The reason that ${SUM}_\beta$ learns the gradient is because stationary points need to satisfy $\mathbf{D}_T\mathbf{h}(\hat{x})\approx 0=\sum_j\nabla \varphi(\hat{x}-x_j)$. It should be noted that the first maps for all $\beta$ in Fig.~\ref{fig:meanMedian}i are different from their $T^{th}$ maps. This is because the first maps try to move $\hat{x}_0^{(i)}$ as close to $\hat{x}_*^{(i)}$ as possible and thus disregard the placement of the stationary point. The training errors in Fig.~\ref{fig:meanMedian}g show that convex functions are easier to learn than nonconvex ones. This is because nonconvex functions may have multiple local minima, which means there may not exist an ideal map where all training data are monotone at their solutions, thus $\hat{x}^{(i)}_t$ may get stuck at a wrong stationary point. Fig.~\ref{fig:meanMedian}h shows that the map have decreasing norms, which represents reducing step sizes as the estimates approach the solutions. 

We also perform an experiment on unseen sets of data, where we compare the global minimizer of $P_\beta$ of each test data $X$ against the solution from \texttt{fminunc} (quasi-Newton) of all $P_\omega,\omega=1,\dots,6$ and the solution of ${SUM}_\beta$. Table~\ref{table:meanMedian} show the MAE over 1000 test sets. We can see that DO can approximate the solution better than using incorrect cost functions. An interesting point to note is that DO seems to be able to solve nonconvex problems better than \texttt{fminunc}, suggesting DO can avoid some local minima and more often terminate closer to the global minimum.

We summarize this section in 4 points. \textit{(i)} We show that DO can learn to mimic gradient of unknown penalty functions. \textit{(ii)} A very important point to note is that a single training data can have multiple ground truths, and DO will learn to find the solution based on the ground truths provided during the training. Thus, it is unreasonable to use DO that, say, trained with the mean as ground truth and hope to get the median as a result. \textit{(iii)} A practical implication of this demonstration is that if we optimize a wrong cost function then we may obtain a bad optimum as solution, and it can be more beneficial to obtain training data and learn to solve for the solution directly. \textit{(iv)} We show that for nonconvex problems, DO has the potential to skip local minima and arrive at a better solution than that of \texttt{fminunc}.

\begin{table}
\caption{\label{table:meanMedian}MAE for solving unknown cost functions. Best results in underline bold, and second best in bold.}
\centering
\begin{tabular}{|c|c|c|c|c|c|c|c|c|}
\hline 
\multirow{2}{*}{$P_\beta$} & \multicolumn{6}{c|}{\texttt{fmincon}} & \multirow{2}{*}{$SUM_\beta$}\tabularnewline
\cline{2-7} 
 & $P_1$ & $P_2$ & $P_3$ & $P_4$ & $P_5$ & $P_6$ & \tabularnewline
\hline
$P_1$ & \underline{\textbf{.0000}} & .0675 & .1535 & .0419 & .0707 & .2044 & \textbf{.0137}
\tabularnewline\hline
$P_2$ & .0675 & \underline{\textbf{.0000}} & .1445 & .1080 & .1078 & .2628 & \textbf{.0145}
\tabularnewline\hline
$P_3$ & .1535 & .1445 & \underline{\textbf{.0000}} & .1743 & .1657 & .2900 & \textbf{.0086} 
\tabularnewline\hline
$P_4$ & .0493 & .1009 & .1682 & \textbf{.0457} & .0929 & .1977 & \underline{\textbf{.0325}} 
\tabularnewline\hline
$P_5$ & .0707 & .1078 & .1657 & .0823 & \underline{\textbf{.0000}} & .1736 & \textbf{.0117} 
\tabularnewline\hline
$P_6$ & .2098 & .2515 & .2791 & .1905 & .2022 & \textbf{.1161} & \underline{\textbf{.0698 }}
\tabularnewline\hline

\end{tabular}

\end{table}

\subsection{3D point cloud registration}\label{exp:3Dregis}
In this section, we perform experiments on the task of 3D point cloud registration. The problem can be stated as follows: Let $\mathbf{M}\in\mathbb{R}^{3\times N_M}$ be a matrix containing 3D coordinates of one shape (`model') and $\mathbf{S}\in\mathbb{R}^{3\times N_S}$ for the second shape (`scene'), find the rotation and translation that registers $\mathbf{S}$ to $\mathbf{M}$. Here, we briefly describe our parametrization and experiments. For more details, please see~\cite{JV_CVPR2017}. 

\subsubsection{DO parametrization and training}
We use Lie Algebra~\cite{Hall_Lie04} to parametrize $\mathbf{x}$, which represents rotation and translation, because it provides a linear space with the same dimensions as the degrees of freedom of our parameters. For $\mathbf{h}$, we design it as a histogram that indicates the weights of scene points on the `front' and the `back' sides of each model point  based on its normal vector. Let $\mathbf{n}_a\in\mathbb{R}^3$ be a normal vector of the model point $\mathbf{m}_a$ computed from its neighbors; $\mathcal{T}(\mathbf{y};\mathbf{x})$ be a function that applies rigid transformation with parameter $\mathbf{x}$ to vector $\mathbf{y}$; $S_{a}^{+}=\{\mathbf{s}_b:\mathbf{n}_a^{\top}(\mathcal{T}(\mathbf{s}_{b};\mathbf{x})-\mathbf{m}_a)>0\}$ be the set of scene points on the `front' of $\mathbf{m}_a$; and $S_{a}^{-}$ contains the remaining scene points. We define $\mathbf{h}:\mathbb{R}^6\times\mathbb{R}^{3\times N_S}\rightarrow\mathbb{R}^{2N_M}$ as:
\begin{equation}\label{eq:feature1}
\small
[\mathbf{h}(\mathbf{x};\mathbf{S})]_{a}=\frac{1}{z}\sum_{\mathbf{s}_b\in S_{a}^{+}}\exp\left(\frac{1}{\sigma^2}\Vert\mathcal{T}(\mathbf{s}_b;\mathbf{x})-\mathbf{m}_a\Vert^2\right),
\end{equation}
\begin{equation}\label{eq:feature2}
\small
[\mathbf{h}(\mathbf{x};\mathbf{S})]_{a+N_M}=\frac{1}{z}\sum_{\mathbf{s}_b\in S_{a}^{-}}\exp\left(\frac{1}{\sigma^2}\Vert\mathcal{T}(\mathbf{s}_b;\mathbf{x})-\mathbf{m}_a\Vert^2\right),
\end{equation}
where $z$ normalizes $\mathbf{h}$ to sum to $1$, and $\sigma$ controls the width of the $\exp$ function. $\mathbf{h}$ can be precomputed (see~\cite{JV_CVPR2017}).

Given a model shape $\mathbf{M}$, we first normalized the data to lie in $[-1,1]$, and generated the scene models as training data by uniformly sampling with replacement 400 to 700 points from $\mathbf{M}$. Then, we applied the following perturbations:  \emph{(i) Rotation and translation:} We randomly rotated the model within 85 degrees, and added a random translation in $[-0.3,0.3]^3$. These transformations were used as the ground truth $\mathbf{x}_*$, with $\mathbf{x}_0=\mathbf{0}_6$ as the initialization. \emph{(ii) Noise and outliers:} Gaussian noise with standard deviation $0.05$ was added to the sample. Then we added two types of outliers: sparse outliers (random 0 to 300 points within $[-1,1]^3$); and structured outliers (a Gaussian ball of 0 to 200 points with the standard deviation of 0.1 to 0.25). Structured outliers is used to mimic other dense object in the scene. \emph{(iii) Incomplete shape:} We used this perturbation to simulate self occlusion and occlusion by other objects. This was done by uniformly sampling a 3D unit vector $\mathbf{u}$, then projecting all sample points to $\mathbf{u}$, and removed the points with the top 40\% to 80\% of the projected values. For all experiments, we generated 30000 training samples, and trained a total of $K=30$ maps for SUM with $\lambda=3\times 10^{-4}$ in~\eqref{eq:learningReg} and $\sigma^2=0.03$ in~\eqref{eq:feature1} and~\eqref{eq:feature2}, and set the maximum number of iterations to 1000.

\subsubsection{Baselines and evaluation metrics}
We compared DO with two point-based approaches (ICP~\cite{ICP_PAMI92} and IRLS~\cite{IRLS_COA2014}) and two density-based approaches (CPD~\cite{CPD_PAMI10} and GMMReg~\cite{GMM_PAMI11}). The codes for all methods were downloaded from the authors' websites, except for ICP where we used MATLAB's implementation. For IRLS, the Huber cost function was used. 

We used the registration success rate and the computation time as performance metrics. We considered a registration to be successful when the mean $\ell_2$ error between the registered model points and the corresponding model points at the ground truth orientation was less than $0.05$ of the model's largest dimension.

\subsubsection{Synthetic data}

We performed synthetic experiments using the Stanford Bunny model~\cite{StanfordBunny} (see Fig.~\ref{fig:3DSynthResult}). We used MATLAB's \texttt{pcdownsample} to select 472 points from 36k points as the model $\mathbf{M}$. We evaluated the performance of the algorithms by varying five types of perturbations:
\textit{(i)} the number of scene points ranges from 100\textasciitilde4000 [default = 200\textasciitilde600]; \textit{(ii)} the standard deviation of the noise ranges between 0\textasciitilde0.1 [default = 0]; \textit{(iii)} the initial angle from 0 to 180 degrees [default = 0\textasciitilde60]; \textit{(iv)} the number of outliers from 0\textasciitilde600 [default = 0]; and \textit{(v)} the ratio of incomplete scene shape from 0\textasciitilde0.7 [default = 0]. While we perturbed one variable, the values of the other variables were set to the default values. 
Note that the scene points were sampled from the original 36k points, not from $\mathbf{M}$. All generated scenes included random translation within $[-0.3,0.3]^3$. A total of 50 rounds were run for each variable setting. Training time for DO was 236 seconds (incl. training data generation and precomputing features).

Examples of test data and the results are shown in Fig.~\ref{fig:3DSynthResult}. ICP required low computation time for all cases, but it had low success rates because it tends to get trapped in the local minimum closest to its initialization. CPD generally performed well except when number of outliers was high, and it required a high computation time. IRLS was faster than CPD, but it did not perform well with incomplete targets. GMMReg had the widest basin of convergence but did not perform well with incomplete targets, and it required long computation time for the annealing steps. For DO, its computation time was much lower than those of the baselines. Notice that DO required higher computation time for larger initial angles since more iterations were required to reach a stationary point. In terms of the success rate, we can see that DO outperformed the baselines in almost all test scenarios. This result was achievable because DO does not rely on any specific cost functions, which generally are modelled to handle a few types of perturbations. On the other hand, DO \emph{learns} to cope with the perturbations from training data, allowing it to be significantly more robust than other approaches.

\begin{figure*}[!t]
\centering
\includegraphics[width=6.8in]{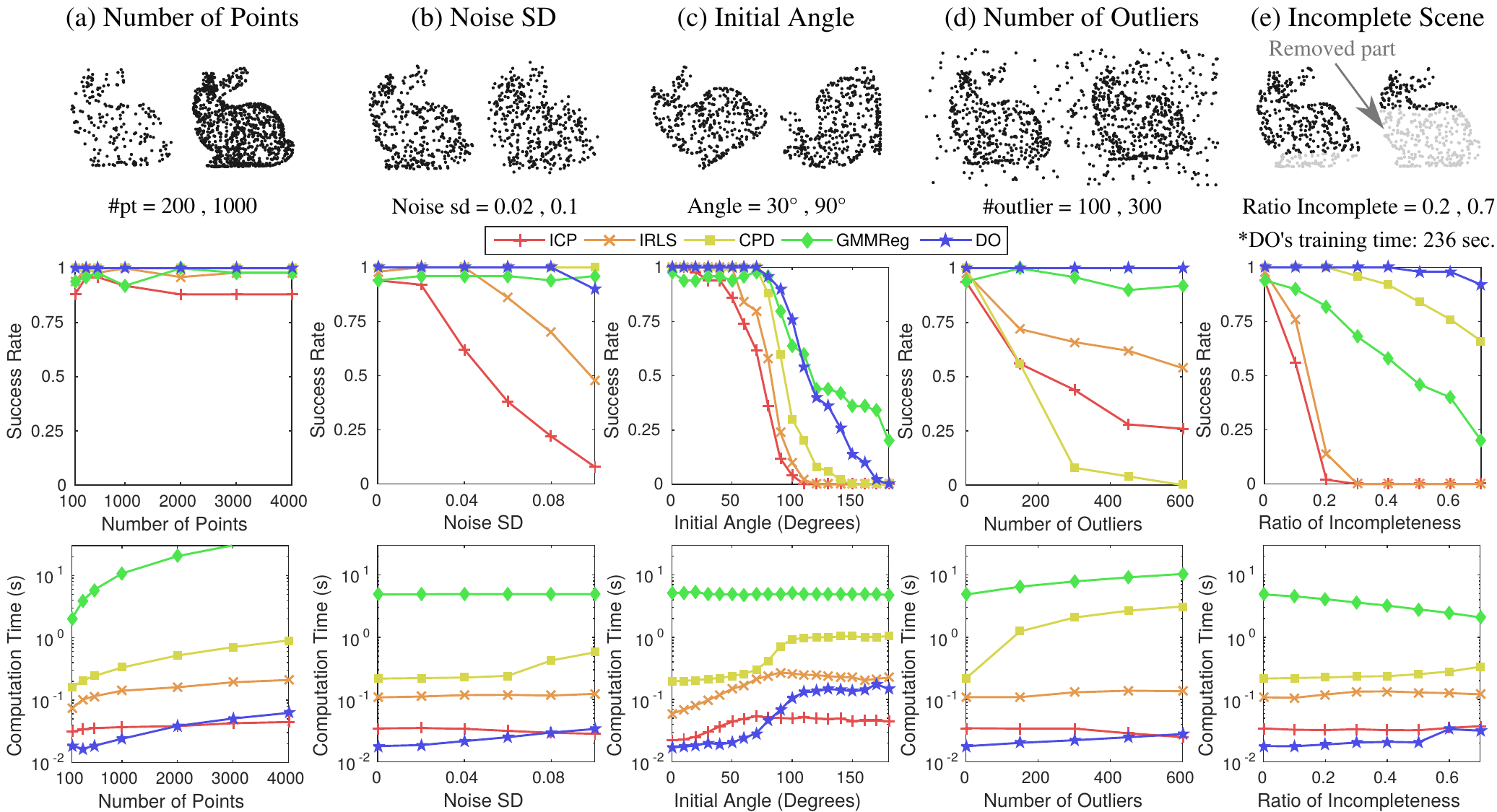}%
\caption{Results of 3D registration with synthetic data under different perturbations. (Left) Examples of scene points with different perturbations. (Middle) Success rate. (Right) Computation time.}
\label{fig:3DSynthResult}
\end{figure*}

\begin{figure}[!t]
\centering
\includegraphics[width=3.4in]{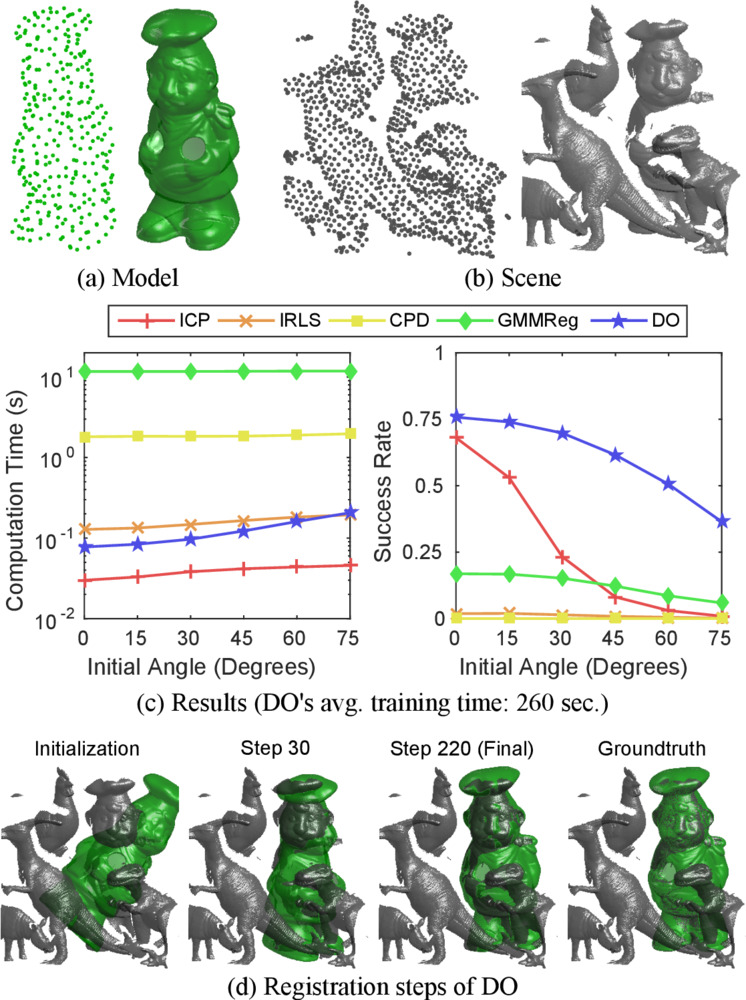}%
\caption{Results of 3D registration with range scan data. (a) example 3D model (`chef'). (b) Example of a 3D scene. We include surface rendering for visualization purpose. (c) Results of the experiment. (d) shows an example of registration steps of DO. The model was initialized 60 degrees from the ground truth orientation with parts of the model intersecting other objects. In addition, the target object is under 70\% occlusion, making this a very challenging case. However, as iteration progresses, DO is able to successfully register the model.}
\label{fig:3DRealResult}
\end{figure}

\subsubsection{Range-scan data}\label{sec:rangeScanExp}
In this section, we performed 3D registration experiment on the UWA dataset~\cite{Mian_IJCV2010}. This dataset contains 50 cluttered scenes with 5 objects taken with the Minolta Vivid 910 scanner in various configurations. All objects are heavily occluded (60\% to 90\%). We used this dataset to test our algorithm under unseen test samples and structured outliers, as opposed to sparse outliers in the previous section. The dataset includes 188 ground truth poses for four objects. We performed the test using all the four objects on all 50 scenes. From the original model, $\sim$300 points were sampled by \texttt{pcdownsample} to use as $\mathbf{M}$ (Fig.~\ref{fig:3DRealResult}a). We also downsampled each scene to $\sim$1000 points (Fig.~\ref{fig:3DRealResult}b). We initialized the model from 0 to 75 degrees from the ground truth orientation with random translation within $[-0.4,0.4]^3$. We ran 50 initializations for each parameter setting, resulting in a total of $50\times 188$ rounds for each data point. Here, we set the inlier ratio of ICP to $50\%$ as an estimate for self-occlusion. Average training time for DO was 260 seconds for each object model.

The results and examples for the registration with DO are shown in Fig.~\ref{fig:3DRealResult}c and Fig.~\ref{fig:3DRealResult}d, respectively. IRLS, CPR, and GMMReg has very low success in almost every scene. This was because structured outliers caused many regions to have high density, creating false optima for CPD and GMMReg which are density-based approaches, and also for IRLS which is less sensitive to local minima than ICP. When initialized close to the solution, ICP could register fast and provided some correct results because it typically terminated at the nearest--and correct--local minimum. On the other hand, DO provided a significant improvement over ICP, while maintaining low computation time. We emphasize that DO was trained with synthetic examples of a single object and it had never seen other objects from the scenes. This experiment shows that we can train DO with synthetic data, and apply it to register objects in real challenging scenes.

\begin{figure}
\centering
\includegraphics[width=3.4in]{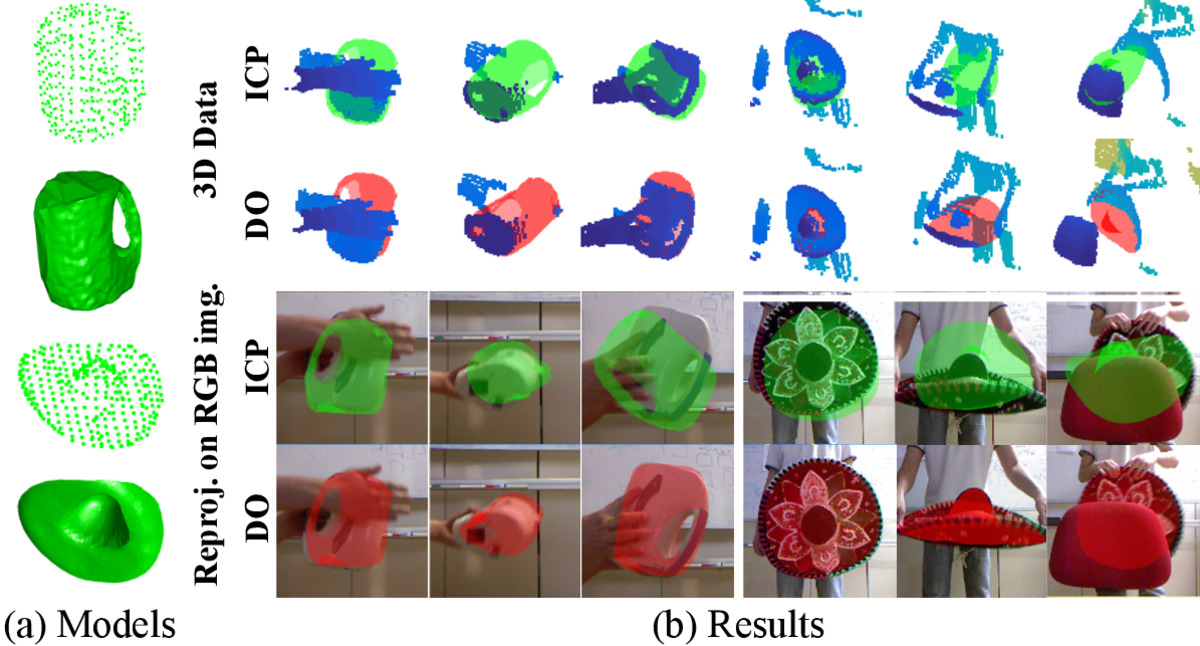}%
\caption{Result for object tracking in 3D point cloud. (a) shows the 3D models of the kettle and the hat. (b) shows tracking results of DO and ICP in (top) 3D point clouds with the scene points in blue, and (bottom) as reprojection on RGB image. Each column shows the same frame.}
\label{fig:3DTrackResult}
\end{figure}

\subsubsection{Application to 3D object tracking}
In this section, we explore the use of DO for 3D object tracking in 3D point clouds. We used Microsoft Kinect to capture RGBD videos at 20fps, then reconstruct 3D scenes from the depth images. We used two reconstructed shapes, a kettle and a hat, as the target objects. These two shapes present several challenges besides self occlusion : the kettle has a smooth surface with few features, while the hat is flat, making it hard to capture from some views. We recorded the objects moving through different orientations, occlusions, etc. The depth images were subsampled to reduce computation load. To perform tracking, we manually initialized the first frame, while subsequent frames were initialized using the pose in the previous frames. Here, we only compared DO against ICP because IRLS gave similar results to those of ICP but could not track rotation well, while CPD and GMMReg failed to handle structured outliers in the scene (similar to Sec.~\ref{sec:rangeScanExp}). Fig.~\ref{fig:3DTrackResult}b shows examples of the results. It can be seen that DO can robustly track and estimate the pose of the objects accurately even under heavy occlusion and structured outliers, while ICP tended to get stuck with other objects. The average computation time for DO was 40ms per frame. This shows that DO can be used as a robust real-time object tracker in 3D point cloud.

\emph{Failure case}: We found DO failed to track the target object when the object was occluded at an extremely high rate, and when the object moved too fast. When this happened, DO would either track another nearby object or simply stay at the same position as in the previous frame.

\subsection{Camera Pose Estimation}
The goal of camera pose estimation is to estimate the relative pose between a given 3D and 2D correspondence set. Given $\{(\mathbf{p}_j,\mathbf{s}_j)\}_{j=1}^J\subset\mathbb{R}^2\times\mathbb{R}^3$ where $\mathbf{p}_j$ is 2D image coordinate and $\mathbf{s}_j$ is the corresponding 3D coordinate of feature $j$, we are interested in estimating the rotation matrix $\mathbf{R}\in SO(3)$ and translation vector $\mathbf{t}\in\mathbb{R}^3$, such that
$$
\tilde{\mathbf{p}}_j\equiv\mathbf{K}\left[\begin{array}{cc}
\mathbf{R} & \mathbf{t}\end{array}\right]\tilde{\mathbf{s}}_{j},j=1,\dots,J,
$$
where tilde denotes homogeneous coordinate, $\mathbf{K}\in\mathbb{R}^{3\times 3}$ is a known intrinsic matrix, and $\equiv$ denotes equivalence up to scale. General approaches for camera pose estimation involve solving nonlinear problems~\cite{Lepetit_IJCV2008,Zheng_ICCV2013,Kneip_CVPR2011,Li_PAMI2012}. Most of existing approaches assume that there are no outlier matches in the correspondence set. When outliers are present, they rely on RANSAC~\cite{Fischler_RANSAC1981} to select the inliers. One approach that does not rely on RANSAC is REPPnP~\cite{Ferraz_CVPR2014}. It finds the camera pose by solving for the robust nullspace of a matrix which represents algebraic projection error. In this section, we will use DO to find a set of inliers, then postprocess the inliers to obtain the camera pose. We show that our algorithm is more robust than REPPnP while being faster than RANSAC-based approaches.

\begin{figure*}[t]
\centering
\includegraphics[width=6.86in]{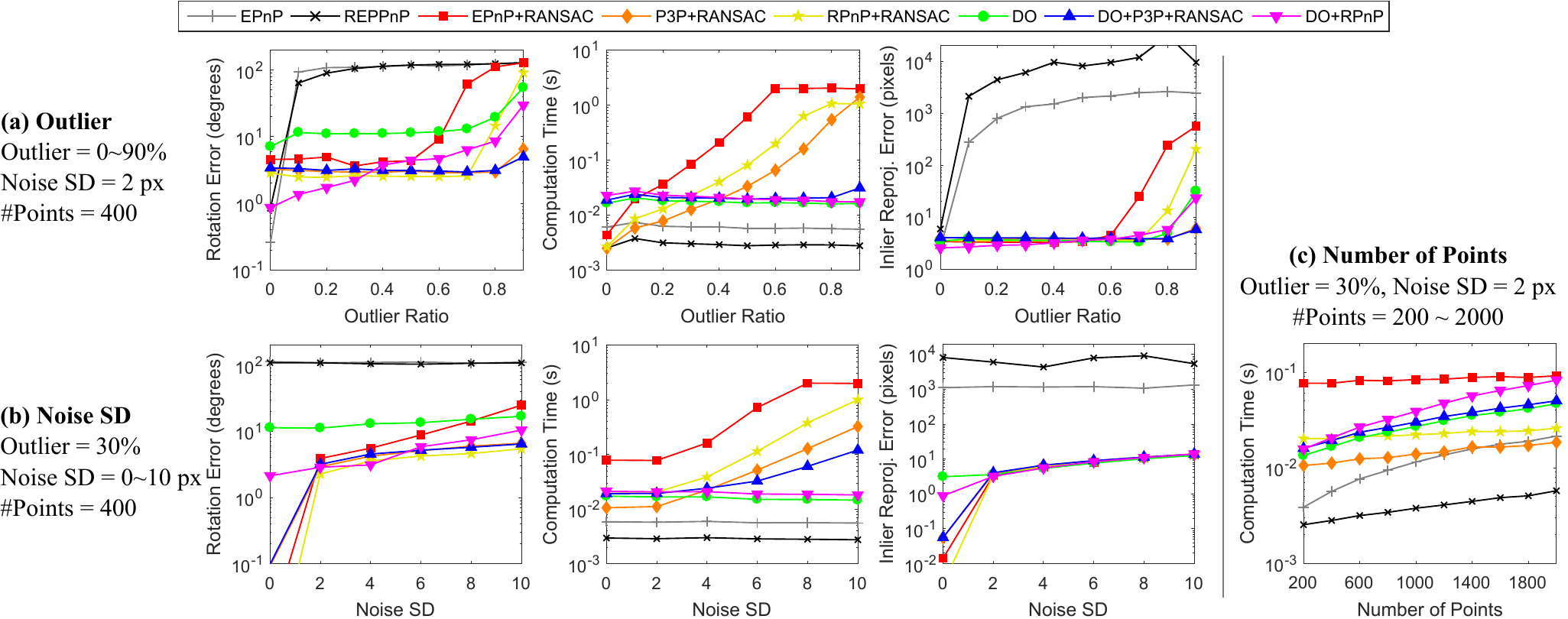}%
\caption{Results for PnP with synthetic data. Varying parameters are (a) outlier ratio, (b) noise SD, and (c) number of points.}
\label{fig:pnpResultSynth}
\end{figure*}

\begin{figure}
\centering
\includegraphics[width=3.4in]{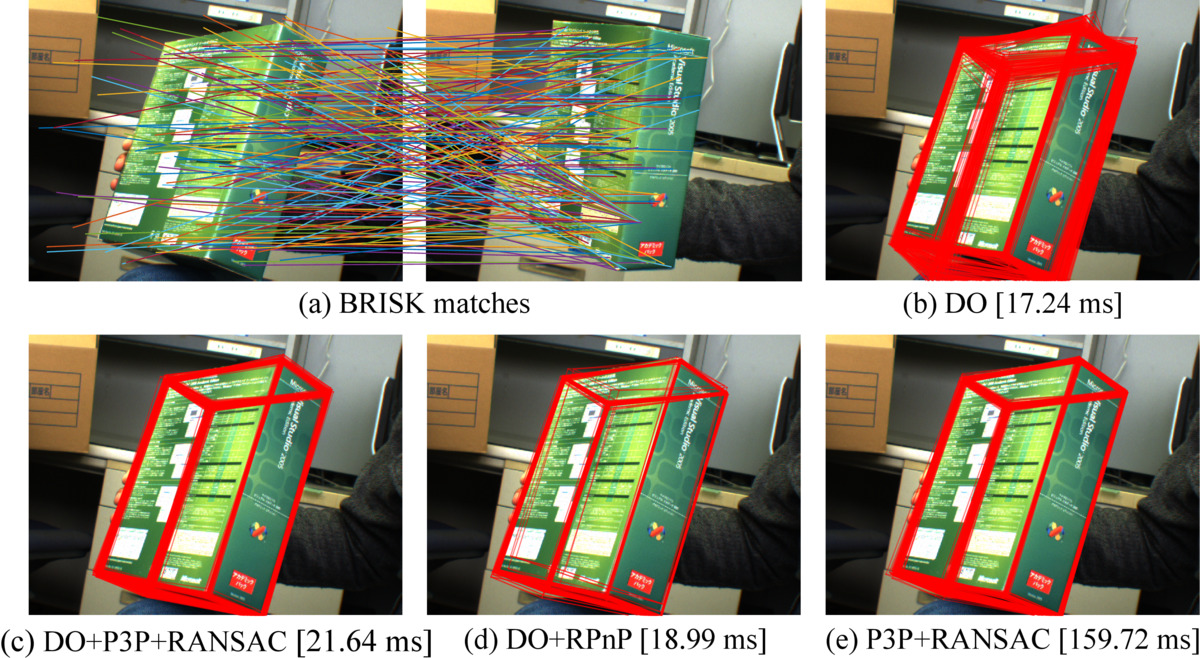}%
\caption{Results for camera pose estimation on real image. (a) Feature matches. Left and right images contain 2D points and projection of 3D points,~\emph{resp}. (b-d) Projected shape with average time over 100 trials.}
\label{fig:pnpResultReal}
\end{figure}

\subsubsection{DO parametrization and training}
To obtain the set of inliers, we solve for a matrix $\mathbf{X}=[\mathbf{x}_1,\mathbf{x}_2, \mathbf{x}_3]^\top\in\mathbb{R}^{3 \times 4}$ such that the geometric error~\cite{MultiviewGeom_2003} is zero (assuming $\mathbf{p}_j$ is calibrated): 
\begin{equation}
\mathbf{g}_j(\mathbf{X}) = 
\mathbf{p}_j-
\left[
\begin{array}{c}
\mathbf{x}_1^\top\tilde{\mathbf{s}}_j/
\mathbf{x}_3^\top\tilde{\mathbf{s}}_j\\
\mathbf{x}_2^\top\tilde{\mathbf{s}}_j/
\mathbf{x}_3^\top\tilde{\mathbf{s}}_j
\end{array}
\right]
=0
,j=1,\dots,J.
\end{equation}
The optimization for solving $\mathbf{X}$ is formulated by summing the error over the correspondences:
\begin{equation}
\minimize_X \frac{1}{J}\sum_{j=1}^J \varphi(\mathbf{g}_j(\mathbf{X})),
\end{equation}
where $\varphi$ is a penalty function. Following the derivation in Sec.~\ref{sec:designH}, the $\mathbf{h}$ function can be derived as:
\begin{align}
\mathbf{h}(\mathbf{X}) & = \frac{1}{J}\sum_{j=1}^J \bigoplus_{l=1}^{12} \bigoplus_{k=1}^2 \left[\frac{\partial \mathbf{g}_j(\mathbf{X})}{\partial vec(\mathbf{X})}\right]_{lk}\bigotimes_{\alpha=1}^2\mathbf{e}_{\delta([\mathbf{g}_j(\mathbf{X})]_{\alpha})}.\label{eq:featPnP}
\end{align}
After computing~\ref{eq:featPnP}, we normalize it to a unit vector and use it our feature. Note that, although the Jacobian matrix of $\mathbf{g}_j$ is a $12 \times 2$ matrix, it has only 12 degrees of freedom. Thus, we need to consider only its 12 values instead of all 24. 

We generated DO's training data as follows. Each image was assumed to be 640 by 480 pixels. \textit{Generating 3D shapes}: A 3D shape, composing of 100 to 500 points, was generated as random points in one of the following shapes: \textit{(i)} in a box; \textit{(ii)} on a spherical surface; and \textit{(iii)} on multiple planes. For \textit{(iii)}, we randomly generated normal and shift vectors for 2 to 4 planes, then added points to them. All shapes were randomly rotated, then normalized to fit in $[-1,1]^3$. \textit{Generating camera matrix}: We randomized the focal length in $[600,1000]$ with the principal point at the center of the image. We sampled the rotation matrix from $SO(3)$, while the translation was generated such that the projected 3D points lie in the image boundary. \textit{Generating image points}: We first projected the 3D shape using the generated camera parameters, then randomly selected 0\% to 80\% of the image points as outliers by changing their coordinates to random locations. All random numbers were uniformly sampled. No noise is added for the training samples. To reduce the effect of varying sizes of images and 3D points, we normalized the inputs to lie in $[-0.5,0.5]$.\footnote{Camera matrix needs to be transformed accordingly, similar to~\cite{hartley_PAMI1997}.} Since the camera matrix is homogeneous, we normalize it to have a unit Frobenius norm. We use $[-1,1]$ as the range for each dimension of $\mathbf{g}_j$, and discretize it to 10 bins. We generated 50000 training samples, and trained 30 maps with $\lambda=10^{-4}$. The training time was 252 seconds.

We compared 3 DO-based approaches: \textit{DO}, \textit{DO+P3P+RANSAC}, and \textit{DO+RPnP}. When DO returned $\mathbf{x}$ as result,  we transformed it back to a $3\times 4$ matrix $\mathbf{M}$, then projected the first three columns to obtain the rotation matrix. For DO+P3P+RANSAC and DO+RPnP, we used $\mathbf{M}$ from DO to select matches with small projection errord as inliers, then calculated the camera parameters using P3P+RANSAC~\cite{Kneip_CVPR2011} and RPnP~\cite{Li_PAMI2012} (without RANSAC). 

\subsubsection{Baselines and evaluation metrics}\label{sec:PnPbaseline}
We compared our approach against 5 baselines. EPnP~\cite{Lepetit_IJCV2008} and REPPnP~\cite{Ferraz_CVPR2014} are deterministic approaches. The other three baselines, P3P+RANSAC~\cite{Kneip_CVPR2011}, RPnP+RANSAC~\cite{Li_PAMI2012}, and EPnP+RANSAC~\cite{Lepetit_IJCV2008} rely on RANSAC to select inliers and use the respective PnP algorithms to find the camera parameters. The minimum number of matches for each algorithm is 3, 4, and 6, \textit{resp}. We use the code from~\cite{Ferraz_CVPR2014} as implementation of the PnP algorithms. The RANSAC routine automatically determines the number of iterations to guarantee 99\% chance of obtaining the inlier set. The performance are measured in terms of \textit{(i)} mean computation time, \textit{(ii)} mean rotation angle error, and \textit{(iii)} mean inlier reprojection error. 

\subsubsection{Experiments and results}
We first performed experiments using synthetic data. We generated the test data using the same approach as training samples. We vary 3 parameters: \textit{(i)} the number of points from 200\textasciitilde2000 [default = 400]; \textit{(ii)} the ratio of outliers from 0\%\textasciitilde90\% [default = 30\%]; and \textit{(iii)} the noise standard deviation from 0\textasciitilde10 pixels [default = 2]. When one parameter is varied, the other two parameters were set to the default values. We performed a total of 500 trials for each setting. 

Fig.~\ref{fig:pnpResultSynth} shows the results of the experiments. In Fig.~\ref{fig:pnpResultSynth}a, we can see that RANSAC-based approaches could obtain accurate results, but their computation time grows exponentially with the outlier ratio. On the other hand, EPnP and REPPnP which are deterministic performed very fast, but they are not robust against outliers even at 10\%. For our approaches, it can be seen that DO alone did not obtain good rotations since it did not enforce any geometric constraints. However, DO could accurately align the 3D inlier points to their image points as can be seen by its low inlier reprojection errors. This is a good indication that DO can be used for identifying inliers. By using this strategy, DO+P3P+RANSAC could obtain accurate rotation up to 80\% of outliers while maintaining low computation time. In contrast, DO+RPnP could obtain very accurate rotation when there were small outliers, but the error increases as it was easier to mistakenly include outliers in the post-DO step. For the noise case (Fig.~\ref{fig:pnpResultSynth}b), DO+RPnP has constant time for all noise levels and could comparatively obtain good rotations under all noise levels, while DO+P3P+RANSAC required exponentially increasing time as points with very high noise may be considered as outliers. Finally, in Fig.~\ref{fig:pnpResultSynth}c, we can see that computation times of all approaches grow linearly with the number of points, but those of DO approaches grow with faster rate, which is a downside of our approach.

Next, we performed experiments on real images. We used an image provided with the code in~\cite{Zheng_ICCV2013}. Fig.~\ref{fig:pnpResultReal}a shows the input matches. Notice that the matches are not one-to-one. Although DO is a deterministic algorithm, different configurations of the same 3D shape can affect the result. For example, we might consider either a 3D shape or its 90$^{\circ}$ rotated shape as the initial configuration with the identity transformation. To measure this effect, we performed 100 trials for DO-based algorithms, where we randomly rotate the 3D shape as the initial configuration. Similarly, we performed 100 trials for P3P+RANSAC. ~\ref{fig:pnpResultReal}b-e show the results. It can be seen that DO can gives a rough estimate of the camera pose, then DO+P3P+RANSAC and DO+RPnP can postprocess to obtain accurate pose. P3P+RANSAC also obtained the correct pose, but it required 8 times the computation time of DO-based approaches. (More results provided in the appendex.)

Since DO is a learning-based approach, the main limitation of DO is that it may not work well with data that are not represented in training, e.g., when the depths and perturbations of training data and test data are different. It is not simple to generate training data to cover all possible cases. On the other hand, PnP solvers of RANSAC-based approaches can reliably obtain the correct pose since they directly solve the geometric problem.

\subsection{Image Denoising}
In this final experiment, we demonstrate the potential of DO for image denoising. This experiments serves to illustrate the potential of DO in multiple ways. First, we illustrate that an SUM trained in a simple fashion can compare favorably against state-of-the-art total variation (TV) denoising algorithms for impulse noises. Second, we show that a SUM can be used to estimate a large and variable number of parameters (number of pixels in this case). This differs from previous experiments that used DO to estimate a small, fixed number of parameters. Third, we show that it is simple for DO to incorporate additional information, such as intensity mask, during both training and testing. Finally, we demonstrate the effect of training data on the results.

\subsubsection{DO parametrization and training}
We based our design of $\mathbf{h}$ on the TV denoising model~\cite{Chan_TV2006}, where we replace the penalty functions on both the data fidelity term and the regularization term with unknown functions $\varphi_1$ and $\varphi_2$:
\begin{align}
\minimize_{\{x_i\}} & \sum_{i\in\Omega} \left(m_i\varphi_1(x_i-u_i) + \sum_{j\in\mathcal{N}(i)} \varphi_2(x_i-x_j)\right),
\end{align}
where $\Omega$ is the image support, $u_i\in[0,1]$ is the intensity at pixel $i$ of the noisy input image, $m_i\in\{0,1\}$ is a given mask, and  $\mathcal{N}(i)$ is the set of neighboring pixels of $i$. The goal is to estimate the clean image $\{x_i\}$. 

In order to allow the learned SUM to work with images of different size, we will treat each pixel $i$ independently: Each pixel will have its own estimate $x^i$.\footnote{The idea is similar to parameter sharing in deep neural network.} Since we have two error terms, we follow Sec.~\ref{sec:designH} and concatenate the indicator of the two errors to form $\mathbf{h}$ as 
\begin{equation}
\mathbf{h}(x_i) = \left[m_i\mathbf{e}_{\gamma(x_i-u_i)}^\top,\sum_{j\in\mathcal{N}(i)}\mathbf{e}_{\gamma(x_i-x_j)}^\top\right]^\top.
\end{equation}
The first part of $\mathbf{h}$ accounts for the data fidelity term, while the second part accounts for the regularization term.

In order to train DO, we randomly sample $1000$ patches of size $40\times40$ to $80\times80$ from the training image, then randomly replace $0\%$ to $80\%$ of the pixels with impulse noise to create noisy images. We trained 3 SUMs: \textit{(i)} \textit{DO-SP}, where we used salt-pepper (SP) impulse noise in $\{0,1\}$; \textit{(ii)} \textit{DO-RV}, where we used random-value (RV) impulse noise in $[0,1]$; and \textit{(iii)} \textit{DO-SPRV}, where 50\% of the images has RV noise, while the rest have SP noise.  This is to study the effect of training data on learned SUMs. Following~\cite{Yuan_CVPR2015}, for images with SP noise, we set the mask $m_i=0$ for pixels with intensity $0$ and $1$ and $m_i=1$ for others. For images with RV noise, we set $m_i=1$ for all pixels as we cannot determine whether a pixel is an impulse noise or not. The intensity of each pixel in the noisy image is treated as initial estimate $x_{0}$, and $x_*$ is its noise-free counterpart. We use $[-2,2]$ as the ranges for both $x_i-u_i$ and $x_i-x_j$, and discretize them to 100 bins. We train a total of 30 maps for DO with $\lambda=10^{-2}$. The training time took on average 367 seconds. During test, we use $maxIter=200$ as the stopping criteria.

\begin{figure}
\centering
\includegraphics[width=3.2in]{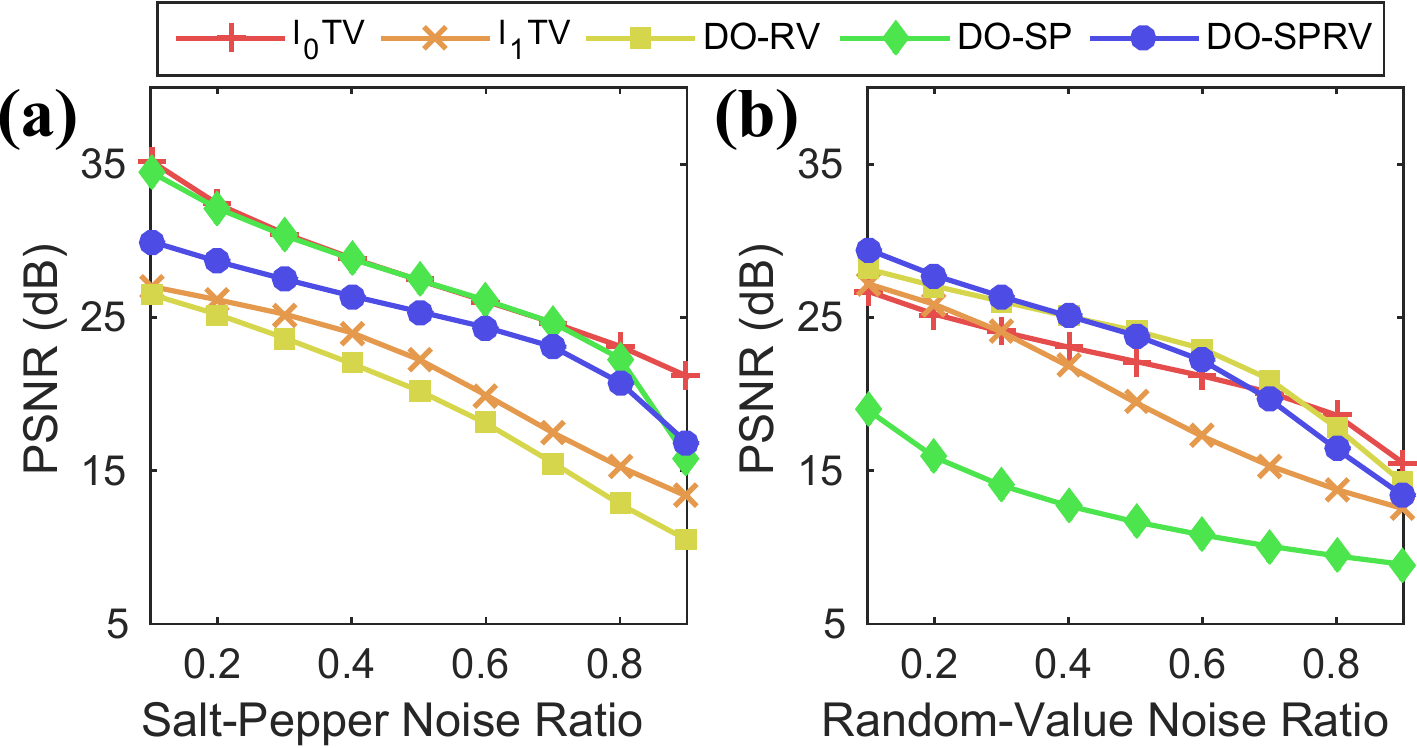}%
\caption{Results for image denoising for (a) salt-pepper impulse noise, and (b) random-value impulse noise.}
\label{fig:exp_denoise_result}
\end{figure}

\subsubsection{Baseline and evaluation metrics}
We compared our approach with two total variation (TV) denoising algorithms which are suitable for impulse noise. The first baseline is the convex $\ell_1TV$~\cite{Chambolle_l1tv2009}, which uses $\ell_1$ for the data fidelity term and isotropic TV as the regularization term. The optimization is solved by the ADMM algorithm. The second baseline is $\ell_0TV$~\cite{Yuan_CVPR2015}, which uses the nonconvex $\ell_0$ for the data term and isotropic TV for the regularization term. The optimization is solved by the Proximal ADMM algorithm. The codes of both algorithms are provided in the toolbox of~\cite{Yuan_CVPR2015}. We used the same mask $m_i$ as in the DO algorithms. We compare the results in terms of Peak Signal-to-Noise Ratio (PSNR). 

\begin{figure*}[!t]
\centering
\includegraphics[width=7in]{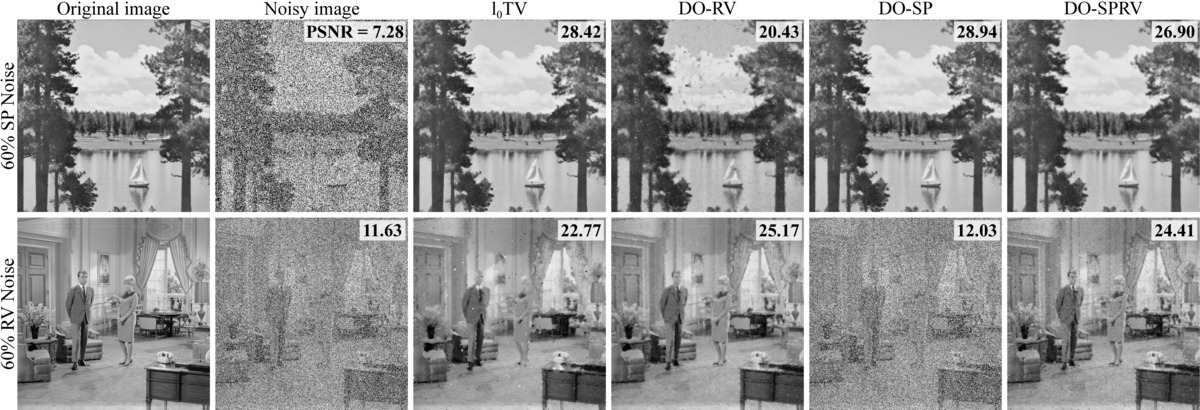}%
\caption{Examples of images denoising results for (top) salt-pepper impulse noise, and (bottom) random-value impulse noise. The PSNR for each image is shown on the top-right. (Best viewed electronically).}
\label{fig:exp_denoise_example}
\end{figure*}

\subsubsection{Experiments and results}

We downloaded 96 grayscale images of size $512\times512$ pixels from the Image Database\footnote{http://decsai.ugr.es/cvg/dbimagenes/g512.php} of University of Granada's Computer Vision Group. 
The first 30 images were used for training DO and selecting the best hyperparameters for the baselines and noise types, while the remaining 66 images were used for evaluation.
For each image, we add impulse noise of $10\%$ to $90\%$ to measure the algorithm robustness.

Fig.~\ref{fig:exp_denoise_result} show the result PNSR over different noise ratios. It can be seen that DO trained with the correct noise type can match or outperform state-of-the-art algorithms, while using a wrong DO give a very bad result. Interestingly, DO-SPRV which was trained with both noise performed well for both cases. Fig.~\ref{fig:exp_denoise_example} shows examples of denoising results of each algorithm ($\ell_1TV$ omitted for clarity of other approaches). For SP noise, $\ell_0TV$, DO-SP, and DO-SPRV can recover small details, while $\ell_1TV$ oversmoothed the image and DO-SP returned an image with smudges. For RV noise, DO-RV returned the best result. DO-SPRV also returned an acceptable image but still contain intensity clumps, while DO-SP cannot recover the image at all. On the other hand, both baselines oversmoothed the image (notice the persons' heads) and still have intensity clumps over the images. This experiment shows that DO can robustly handle different types of impulse noises, and that the training data have a strong effect on types and amount of noise that it can handle. The best approach for solving the problem is to select the correctly trained model. Still, training DO with both noise can return a good result, illustrating the potential of DO in solving a hard problem.

\section{Conclusion and Discussions}
We presented Discriminative Optimization (DO) as a new framework for solving computer vision problem. DO learns a sequence of update maps that update a set of parameters to a stationary point from training data. We study the relationship between DO and optimization-based and learning-based algorithms. A theoretical result on the convergence of the training error is given. We also proposed a framework for designing features for DO based on gradient descent, and provide a theoretical result that DO can learn to mimic gradient descent to solve some optimization problems with unknown cost functions. The result is supported by synthetic experiment that the maps learn the approximation of the gradients. In terms of applications, we show that DO can provide favorable results against state-of-the-art algorithms in the problems of 3D point cloud registration and tracking, camera pose estimation, and image denoising. This also shows that DO can deal with both ordered and unordered data. 

Although algorithms similar to DO have been proposed previously, this paper opens the connection between DO and optimization. Future work may import ideas and intuition from optimization to DO, such as the incorporation of constraints and momentum methods. Also, our theoretical result provides only sufficient conditions, and monotonicity-at-a-point conditions suggest relationship to the variation inequality problems~\cite{Nagurney_VarIneq1999}, which can be considered as a generalization of optimization and which might explain the robustness of DO. We also observe that DO's update rule can be compared with the layer in deep residual network~\cite{he_cvpr2016}, since they both iteratively perform update of the form $\mathbf{x}+T\mathbf{x}$ for some transformation $T$. This can provide some connection to some deep learning algorithms. Future research may also address convergence in the test data or other approaches for designing feature function $\mathbf{h}$. With a strong theoretical foundation and practical potential, we believe DO opens a new exciting research area which would have strong impact to computer vision. 

\bibliographystyle{IEEEtran}
\bibliography{SDM_PAMI}
%
%
%
%
\appendices
\setcounter{theorem}{0}
\section{Proofs for Theoretical Results}\label{sec:proof}
In this section, we provide the proofs for the theoretical results in the main DO paper.
\subsection{Proof of the Thm.~\ref{thm:trainErr}}

\begin{theorem}
{\bf (Convergence of SUM's training error)} Given a training set $\{(\mathbf{x}^{(i)}_0,\mathbf{x}_*^{(i)},\mathbf{h}^{(i)})\}_{i=1}^N$, if there exists a linear map $\mathbf{\hat{D}}\in\mathbb{R}^{p\times f}$
where $\hat{\mathbf{D}}\mathbf{h}^{(i)}$ is strictly monotone at $\mathbf{x}_*^{(i)}$ for all $i$, and if there exists an $i$ where $\mathbf{x}^{(i)}_t\neq\mathbf{x}_*^{(i)}$, then the update rule:
\begin{equation}
\mathbf{x}^{(i)}_{t+1}=\mathbf{x}^{(i)}_{t}-
\mathbf{D}_{t+1}\mathbf{h}^{(i)}(\mathbf{x}^{(i)}_{t}),
\end{equation}
with $\mathbf{D}_{t+1}\subset\mathbb{R}^{p\times f}$
obtained from~(3) in the main paper, guarantees that the training error strictly decreases in each iteration:
\begin{equation}
\sum_{i=1}^{N}
\Vert\mathbf{x}_*^{(i)}-\mathbf{x}^{(i)}_{t+1}\Vert_2^{2}
<
\sum_{i=1}^{N}
\Vert\mathbf{x}_*^{(i)}-\mathbf{x}^{(i)}_t\Vert_2^{2}.
\end{equation}
Moreover, if $\hat{\mathbf{D}}\mathbf{h}^{(i)}$ is strongly monotone at $\mathbf{x}_*^{(i)}$, and if there exist $M>0,H\geq0$ such that $\Vert \hat{\mathbf{D}}\mathbf{h}^{(i)}(\mathbf{x}^{(i)})\Vert_2^2\leq H+M\Vert\mathbf{x}_*^{(i)}-\mathbf{x}^{(i)}\Vert_2^2$ for all $i$, then the training error converges to zero. If $H=0$ then the error converges to zero linearly.
\label{thm:trainErr}
\end{theorem}

\begin{proof}
First, we show the case of strictly monotone at a point. For simplicity, we denote $\mathbf{x}_{t+1}^{(i)}$ and $\mathbf{x}_{t}^{(i)}$ as $\mathbf{x}_{+}^{(i)}$ and $\xii$, respectively. We assume that not all $\xii_*=\xii$, otherwise
all $\xii_*$ are already at their stationary points. Thus, there exists an $i$ such that $(\xii-\xii_*)^{\top}\hat{\mathbf{D}}\hi(\xii)>0$.
We need to show that
\begin{equation}
\sum_{i=1}^{N}\Vert\xii_*-\mathbf{x}_{+}^{(i)}\Vert_{2}^{2}<\sum_{i=1}^{N}\Vert\xii_*-\xii\Vert_{2}^{2}.
\end{equation}
This can be shown by letting $\bar{\mathbf{D}}=\alpha\hat{\mathbf{D}}$
where:
\begin{equation}
\alpha=\frac{\beta}{\gamma},
\end{equation}
\begin{equation}
\beta=\sum_{i=1}^N(\xii-\xii_*)^{\top}\hat{\mathbf{D}}\hi(\xii),
\end{equation}
\begin{equation}
\gamma=\sum_{i=1}^N\Vert\hat{\mathbf{D}}\hi(\xii)\Vert_2^2.
\end{equation}
Since there exists an $i$ such that $(\xii-\xii_*)^{\top}\hat{\mathbf{D}}\mathbf{h}(\xii)>0$, both $\beta$ and $\gamma$ are both positive,
and thus $\alpha$ is also positive. Now, we show that the training error decreases in each iteration as follows:

\begin{align}
\sum_{i=1}^{N}\Vert\xii_*-\mathbf{x}_{+}^{(i)}\Vert_2^2 
& = \sum_{i=1}^{N}\Vert\xii_*-\xii+\mathbf{D}_{t+1}\hi(\xii)\Vert_2^2
\\
 & \leq \sum_{i=1}^{N}\Vert\xii_*-\xii+\bar{\mathbf{D}}\hi(\xii_t)\Vert_2^2\label{eq:strDecEq}
 \\
 & = \sum_{i=1}^{N}\Vert\xii_*-\xii\Vert_2^2 + 
 \\
 & \phantom{{}=1}+ \underbrace{\sum_{i=1}^{N}\Vert\alpha\hat{\mathbf{D}}\hi(\xii)\Vert_2^2}_{\alpha^{2}\gamma} + \notag
 \\
 & \phantom{{}=1} + 2\alpha\underbrace{\sum_{i=1}^{N}(\xii_*-\xii)^{\top}\hat{\mathbf{D}}\hi(\xii)}_{=-\beta} \notag
 \\
 & = \sum_{i=1}^{N}\Vert\xii_*-\xii\Vert_2^2+\alpha^{2}\gamma-2\alpha\beta
 \\
 & = \sum_{i=1}^{N}\Vert\xii_*-\xii\Vert_2^2+\frac{\beta^{2}}{\gamma}-2\frac{\beta^{2}}{\gamma}
 \\
 & = \sum_{i=1}^{N}\Vert\xii_*-\xii\Vert_2^2-\underbrace{\frac{\beta^{2}}{\gamma}}_{>0}\label{eq:proofReferStrict}
 \\
 & < \sum_{i=1}^{N}\Vert\xii_*-\xii\Vert_2^2.\label{eq:strictNonIncrease}
\end{align}
Eq.~\ref{eq:strDecEq} is due to $\mathbf{D}_{t+1}$ being the optimal matrix that minimizes the squared error. Note that Thm.~\ref{thm:trainErr} does not guarantee that the error of each sample $i$ reduces in each iteration, but guarantees the reduction in the average error.

For the case of strongly monotone at a point, we make additional assumption that there exist $H\geq0,M>0$ such that $\Vert\hat{\mathbf{D}}\hi(\mathbf{x}^{(i)})\Vert_2^2\leq H+M\Vert\mathbf{x}_*^{(i)}-\mathbf{x}^{(i)}\Vert_2^2$ for all $\mathbf{x}$ and $i$. Thus, we have 
\begin{equation}
\beta=\sum_{i=1}^N(\xii-\xii_*)^{\top}\hat{\mathbf{D}}\hi(\xii)
\geq m\sum_{i=1}^N\Vert\xii-\xii_*\Vert_2^2,
\end{equation}
\begin{equation}
\gamma=\sum_{i=1}^N\Vert\hat{\mathbf{D}}\hi(\xii)\Vert_2^2
\leq NH+M\sum_{i=1}^N\Vert\xii_*-\xii\Vert_2^2.
\end{equation}
Also, let us denote the training error in iteration $k$ as $E_t$:
\begin{equation}
E_t = \sum_{i=1}^{N}\Vert\xii_*-\xii_t\Vert_2^2.
\end{equation}
From Eq.~\ref{eq:proofReferStrict}, we have
\begin{align}
E_{t+1} & = E_{t}-{\frac{\beta^{2}}{\gamma}}
\\
& \leq E_{t}-
\frac{m^2E_{t}^2}
{NH+ME_t}\label{eq:errorReduceProof}
\\
& =
\left(1-\frac
{m^2E_{t}}
{NH+ME_t}
\right)
E_{t}.
\end{align}
Recursively applying the above inequality, we have
\begin{equation}
E_{t+1} 
\leq
E_{0}
\prod_{l=1}^{t+1}\left(1-\frac{m^2E_{l}}{NH+ME_l}
\right).\label{eq:strongEtPlus1Ineq}
\end{equation}
Next, we will show the following result:
\begin{equation}
\lim_{t\rightarrow\infty}E_{t+1}=0
\end{equation}
This can be shown by contradiction. Suppose $E_t$ converges to some positive number $\mu>0$. Since $\{E_t\}_t$ is a nonincreasing sequence~\eqref{eq:strictNonIncrease}, we have that $E_0>E_t\geq\mu$ for all $t>0$. This means
\begin{equation}
0\leq 1-\frac{m^2E_{t}}{NH+ME_t}<1-\frac{m^2\mu}{NH+ME_0}< 1.\label{eq:strongIneq}
\end{equation}
By recursively multiplying~\eqref{eq:strongIneq}, we have
\begin{align}
\lim_{t\rightarrow\infty}\prod_{l=1}^{t+1}\left(1-\frac{m^2E_{l}}{NH+ME_{l}}\right) 
& \leq \lim_{t\rightarrow\infty}\left(1-\frac{m^2\mu}{NH+ME_0}\right)^{t+1} 
\\
& = 0.\label{eq:strongProdError}
\end{align}
Combining~\eqref{eq:strongProdError} and~\eqref{eq:strongEtPlus1Ineq}, we have
\begin{equation}
\lim_{t\rightarrow\infty}E_{t+1} \leq E_{0} \lim_{t\rightarrow\infty}\prod_{l=1}^{t+1}\left(1-\frac{m^2E_{l}}{NH+ME_{l}}\right) = 0.
\end{equation}
This contradicts our assumption that $\{E_t\}_t$ converges to $\mu>0$. Thus, the training error converges to zero. 

Next, we consider the case where $H=0$. In this case, in~\eqref{eq:errorReduceProof}, we will have
\begin{align}
E_{t+1} & \leq E_{t}-
\frac{m^2E_{t}^2}
{ME_t}
\\
& =
\left(1-\frac
{m^2}
{M}
\right)
E_{t}.
\end{align}
Recursively applying the above inequality, we have
\begin{equation}
E_{t+1} \leq \left(1-\frac{m^2}{M}\right)^{t+1}E_0.
\end{equation}
This proves that the training error converges linearly to zero.
\end{proof}
%

\subsection{Proof of the Prop.~\ref{prop:psMonIsMonAtPt}}
\setcounter{proposition}{1}
\begin{proposition}
(Pseudomonotonicity and monotonicity at a point) If a function $\mathbf{f}:\mathbb{R}^p\rightarrow \mathbb{R}^p$ is pseudomonotone (resp., strictly, strongly) and $\mathbf{f}(\mathbf{x}_*)=\mathbf{0}_p$, then $\mathbf{f}$ is monotone (resp., strictly, strongly) at $\mathbf{x}_*$. 
\label{prop:psMonIsMonAtPt}
\end{proposition}

\begin{proof}
We will show the case of pseudomonotone $\mathbf{f}$. Let $\mathbf{x}'=\mathbf{x}_*$, then we have
\begin{equation}
(\mathbf{x}-\mathbf{x}_*)^\top \mathbf{f}(\mathbf{x}_*) = 0,
\end{equation}
which, by the definition of pseudomonotonicity, implies
\begin{equation}
(\mathbf{x}-\mathbf{x}_*)^\top \mathbf{f}(\mathbf{x}) \geq 0,
\end{equation}
for all $\mathbf{x}\in\mathbb{R}^p$. That means $\mathbf{f}$ is monotone at $\mathbf{x}_*$. The proofs for strict and strong cases follow similar steps.
\end{proof}

\subsection{Proof of Prop.~\ref{prop:DesignH}}
\begin{proposition}
(Convergence of the training error with an unknown penalty function) Given a training set $\{(\mathbf{x}_0^{(i)},\mathbf{x}_*^{(i)},\{\mathbf{g}^{(i)}_j\}_{j=1}^{J_i})\}_{i=1}^N$, where $\mathbf{x}_0^{(i)},\mathbf{x}_*^{(i)}\in\mathbb{R}^{p}$ and $\mathbf{g}_j^{(i)}:\mathbb{R}^{p}\rightarrow\mathbb{R}^d$ differentiable, if there exists a function $\varphi:\mathbb{R}^d\rightarrow\mathbb{R}$ such that for each $i$, $\sum_{j=1}^{J_i} \varphi(\mathbf{g}^{(i)}_j(\mathbf{x}^{(i)}))$ is differentiable strictly pseudoconvex with the minimum at $\mathbf{x}_*^{(i)}$, then the training error of DO with $\mathbf{h}$ from Sec.~5.1 strictly decreases in each iteration. Alternatively, if $\sum_{j=1}^{J_i} \varphi(\mathbf{g}^{(i)}_j(\mathbf{x}^{(i)}))$ is differentiable strongly pseudoconvex with Lipschitz continuous gradient, then the training error of DO converges to zero.\label{prop:DesignH}
\end{proposition}

\begin{proof}
Let $\Phi^{(i)}=\frac{1}{J}\sum_{j=1}^{J_i}\varphi(\mathbf{g}^{(i)}(\mathbf{x}))$. We divide the proof into two cases:

\underline{Case 1}: Differentiable strictly pseudoconvex $\Phi^{(i)}$.

Since $\Phi^{(i)}$ is differentiable strictly pseudoconvex, by Prop.~1, its gradient $\nabla\Phi^{(i)}$ is strictly pseudomonotone. Also, since $\Phi^{(i)}$ has a minimum at $\xii_*$, $\nabla\Phi^{(i)}(\xii_*)=\mathbf{0}_p$. By Prop.~\ref{prop:psMonIsMonAtPt}, this means that $\nabla\Phi^{(i)}$ is monotone at $\xii_*$. If we use $\mathbf{h}$ from Sec.~5.1 and set $\mathbf{\hat{D}}(\mathbf{v},k)$ to $[\mathbf{\phi}(\mathbf{v})]_k$, then we have that $\mathbf{\hat{D}}\mathbf{h}^{(i)}=\nabla\Phi^{(i)}$, meaning $\mathbf{\hat{D}}\mathbf{h}^{(i)}$ is strictly monotone at $\mathbf{x}_*^{(i)}$. Thus, by Thm.~\ref{thm:trainErr}, we have that the training error of DO strictly decreases in each iteration.

\underline{Case 2}: Differentiable strongly pseudoconvex $\Phi^{(i)}$.

The proof is similar to case 1, but differentiable strongly pseudoconvex $\Phi^{(i)}$ will have $\mathbf{\hat{D}}\mathbf{h}^{(i)}=\nabla\Phi^{(i)}$ which is strongly pseudomonotone at $\xii_*$. Since $\nabla\Phi^{(i)}=\mathbf{\hat{D}}\mathbf{h}^{(i)}$ is Lipschitz continuous and $\nabla\Phi^{(i)}(\xii_*)=\mathbf{0}_p$, this means 
\begin{equation}
\Vert\mathbf{\hat{D}}\mathbf{h}^{(i)}(\xii)\Vert_2\leq L\Vert\xii-\xii_*\Vert_2,
\end{equation}  
where $L$ is the Lipschitz constant. Thus, by Thm.~\ref{thm:trainErr}, we have that the training error of DO converges to zero.

\end{proof}

\subsection{Convergence result for nondifferentiable convex $\varphi$}
Here, we provide the convergence result in the case of nondifferentiable convex cost function. Since we will need to refer to subdifferential which is a multivalued map rather than a function, we will need to generalize Prop.~\ref{prop:psMonIsMonAtPt} to the case of multivalued maps. First, we define pseudomontone multivalued maps.
\setcounter{definition}{3}
\begin{definition}
(Pseudomonotone multivalued map~\cite{yao_JOTA1994}) A multivalued map $\mathbf{f}$ is 

(1) \emph{pseudomonotone} if for any distinct points $\mathbf{x},\mathbf{x}'\in\mathbb{R}^p$ and any $\mathbf{u}\in\mathbf{f}(\mathbf{x}),\mathbf{u}'\in\mathbf{f}(\mathbf{x}')$,
\begin{equation}
(\mathbf{x}-\mathbf{x}')^\top \mathbf{u}' \geq 0 \implies (\mathbf{x}-\mathbf{x}')^\top \mathbf{u} \geq 0,
\end{equation}

(2) \emph{strictly pseudomonotone} if for any distinct points $\mathbf{x},\mathbf{x}'\in\mathbb{R}^p$ and any $\mathbf{u}\in\mathbf{f}(\mathbf{x}),\mathbf{u}'\in\mathbf{f}(\mathbf{x}')$,
\begin{equation}
(\mathbf{x}-\mathbf{x}')^\top \mathbf{u}' \geq 0 \implies (\mathbf{x}-\mathbf{x}')^\top \mathbf{u} > 0,
\end{equation}

(3) \emph{strongly pseudomonotone} if there exists $m>0$ such that for any distinct points $\mathbf{x},\mathbf{x}'\in\mathbb{R}^p$ and any $\mathbf{u}\in\mathbf{f}(\mathbf{x}),\mathbf{u}'\in\mathbf{f}(\mathbf{x}')$,
\begin{equation}
(\mathbf{x}-\mathbf{x}')^\top \mathbf{u}' \geq 0 \implies (\mathbf{x}-\mathbf{x}')^\top \mathbf{u} \geq m\Vert\mathbf{x}-\mathbf{x}'\Vert_2^2.
\end{equation}
\end{definition}

Next, we show the following results which generalize Prop.~\ref{prop:psMonIsMonAtPt} to the case of multivalued maps.

\begin{proposition}
(Pseudomonotone multivalued map and monotonicity at a point) If a multivalued map $\mathbf{f}$ is pseudomonotone (resp., strictly, strongly) and $\mathbf{0}_p\in\mathbf{f}(\mathbf{x}_*)$, then $\mathbf{f}$ is monotone (resp., strictly, strongly) at $\mathbf{x}_*$. 
\label{prop:psMonIsMonAtPtMultiMap}
\end{proposition}

\begin{proof}
We will show the case of pseudomonotonicity. Let $\mathbf{x}'=\mathbf{x}_*$. Since $\mathbf{0}_p\in\mathbf{f}(\mathbf{x}_*)$, we have
\begin{equation}
(\mathbf{x}-\mathbf{x}_*)^\top \mathbf{0}_p = 0,
\end{equation}
which, by the definition of pseudomonotonicity, implies for all $\mathbf{u}\in\mathbf{f}(\mathbf{x})$
\begin{equation}
(\mathbf{x}-\mathbf{x}_*)^\top \mathbf{u} \geq 0,
\end{equation}
for all $\mathbf{x}\in\mathbb{R}^p$. That means $\mathbf{f}$ is monotone at $\mathbf{x}_*$. The proofs for strict and strong cases follow similar steps.
\end{proof}

It is known~\cite{Hadjisavvas_EncOpt2001} that monotone (resp., strictly, strongly) multivalued maps are pseudomonotone (resp., strictly, strongly) multivalued maps.

With these results, we can show the convergence result for learning DO under unknown nondifferentiable convex cost functions. 

\begin{proposition}
(Convergence of the training error with an unknown convex penalty function) Given a training set $\{(\mathbf{x}_0^{(i)},\mathbf{x}_*^{(i)},\{\mathbf{g}^{(i)}_j\}_{j=1}^{J_i})\}_{i=1}^N$, where $\mathbf{x}_0^{(i)},\mathbf{x}_*^{(i)}\in\mathbb{R}^{p}$ and $\mathbf{g}_j^{(i)}:\mathbb{R}^{p}\rightarrow\mathbb{R}^d$ differentiable, if there exists a function $\varphi:\mathbb{R}^d\rightarrow\mathbb{R}$ such that for each $i$, $\sum_{j=1}^{J_i} \varphi(\mathbf{g}^{(i)}_j(\mathbf{x}^{(i)}))$ is strictly convex with the minimum at $\mathbf{x}_*^{(i)}$, then the training error of DO with $\mathbf{h}$ from Sec.~5.1 strictly decreases in each iteration. Alternatively, if $\sum_{j=1}^{J_i} \varphi(\mathbf{g}^{(i)}_j(\mathbf{x}^{(i)}))$ is strongly convex with the minimum at $\mathbf{x}_*^{(i)}$ and if there exist $M>0,H\geq 0$ such that 
\begin{equation}
\Vert(1/J)\sum_{j=1}^{J_i} \bar{\varphi}(\mathbf{g}^{(i)}_j(\mathbf{x}^{(i)}))\Vert_2^2\leq H+M\Vert\xii_*-\xii\Vert_2^2 \label{eq:nonDiffCvxAssump}
\end{equation}
for all $i,\xii,\text{ and }\bar{\varphi}(\mathbf{g}^{(i)}_j(\mathbf{x}^{(i)}))\in\partial\varphi(\mathbf{g}^{(i)}_j(\mathbf{x}^{(i)}))$, then the training error of DO converges to zero.\label{prop:DesignHcvx}
\end{proposition}

\begin{proof}
Let $\Phi^{(i)}=\frac{1}{J}\sum_{j=1}^{J_i}\varphi(\mathbf{g}^{(i)}(\mathbf{x}))$. We divide the proof into two cases:

\underline{Case 1}: Strictly convex $\Phi^{(i)}$.

The subdifferential of $\Phi^{(i)}$ is a multivalued map~\cite{Romano_AMO1995}:
\begin{equation}
\partial\Phi^{(i)}=\frac{1}{J}\partial\sum_{j=1}^J\varphi(\mathbf{g}^{(i)}_j(\mathbf{x}))
=
\frac{1}{J}\sum_{j=1}^J\frac{\partial \mathbf{g}^{(i)}_j(\mathbf{x})}{\partial \mathbf{x}}
\partial \varphi(\mathbf{g}^{(i)}_j(\mathbf{x})),
\end{equation}
where $\partial$ denotes subdifferential. Let $\mathbf{\hat{\phi}}$ be a function where $\mathbf{\hat{\phi}}(\mathbf{x})=\mathbf{y}$ for any $\mathbf{y}\in\partial\varphi(\mathbf{x})$. If we use $\mathbf{h}$ from Sec.~5.1 and set $\mathbf{\hat{D}}(\mathbf{v},k)$ to $[\mathbf{\phi}(\mathbf{v})]_k$, we have that $\mathbf{\hat{D}}\mathbf{h}^{(i)}(\mathbf{x})\in\partial\Phi^{(i)}(\mathbf{x})$. 

From~\cite{Rockafellar_cvxAna1970}, we know that $\partial\Phi^{(i)}$ is a strictly monotone map, which means it is also strictly pseudomonotone. Since $\Phi^{(i)}$ has a global minimum at $\xii_*$, then we know that $\partial\Phi^{(i)}$ has zero only at $\xii_*$. Then, by Prop.~\ref{prop:psMonIsMonAtPtMultiMap}, $\mathbf{\hat{D}}\mathbf{h}^{(i)}$ is strictly monotone at $\xii_*$. Thus, by Thm.~\ref{thm:trainErr}, we have that the training error of DO strictly decreases in each iteration.

\underline{Case 2}: Strongly convex $\Phi^{(i)}$.

The proof is similar to case 1, but strongly convex $\Phi^{(i)}$ will have $\mathbf{\hat{D}}\mathbf{h}^{(i)}(\mathbf{x})\in\partial\Phi^{(i)}(\mathbf{x})$ which is strongly monotone at $\xii_*$. By assumption in~\eqref{eq:nonDiffCvxAssump}, we also have that $\Vert\hat{\mathbf{D}}\mathbf{h}^{(i)}(\xii)\Vert_2^2\leq H+M\Vert\xii_*-\xii_t\Vert_2^2$ for all $i,\xii$. Thus, by Thm.~\ref{thm:trainErr}, we have that the training error of DO converges to zero.
\end{proof}

Note that many useful convex functions have subgradients that follow~\eqref{eq:nonDiffCvxAssump}. Examples of such functions include differentiable functions with Lipschitz gradient, e.g., squared $\ell_2$ norm, and functions which are point-wise maximum of a finite number of affine functions, e.g., $\ell_1$ norm. Note that, however, a function which is a point-wise maximum of a finite number of affine functions is not strongly monotone at its minimum since its subgradients are bounded by a constant.

\section{Additional Results for Point Cloud Registration}\label{sec:2Dregis}
In this section, we provide an additional experiment for 2D point cloud registration.
\subsection{2D point cloud registration}
We performed experiments on 2D registration using 4 shapes in Fig.~\ref{fig:model2D} from~\cite{Gold_PR1998,Tsin_ECCV2004}\footnote{Available from http://www.cs.cmu.edu/~ytsin/KCReg/KCReg.zip and http://cise.ufl.edu/$\sim$anand/students/chui/rpm/TPS-RPM.zip}. The shapes were normalized by removing the mean and scaling so that the largest dimension fitted $[-1,1]$. We used the same baselines and performance metrics as in the main paper. Since MATLAB does not provide ICP for 2D case, we used the code of IRLS and set the cost function to least-squared error as ICP. In this 2D experiment, we set the error threshold for successful registration to 0.05 of the model's largest dimension. 

\begin{figure}[!h]
\centering
\hspace{-0.5em}\includegraphics[width=3in]{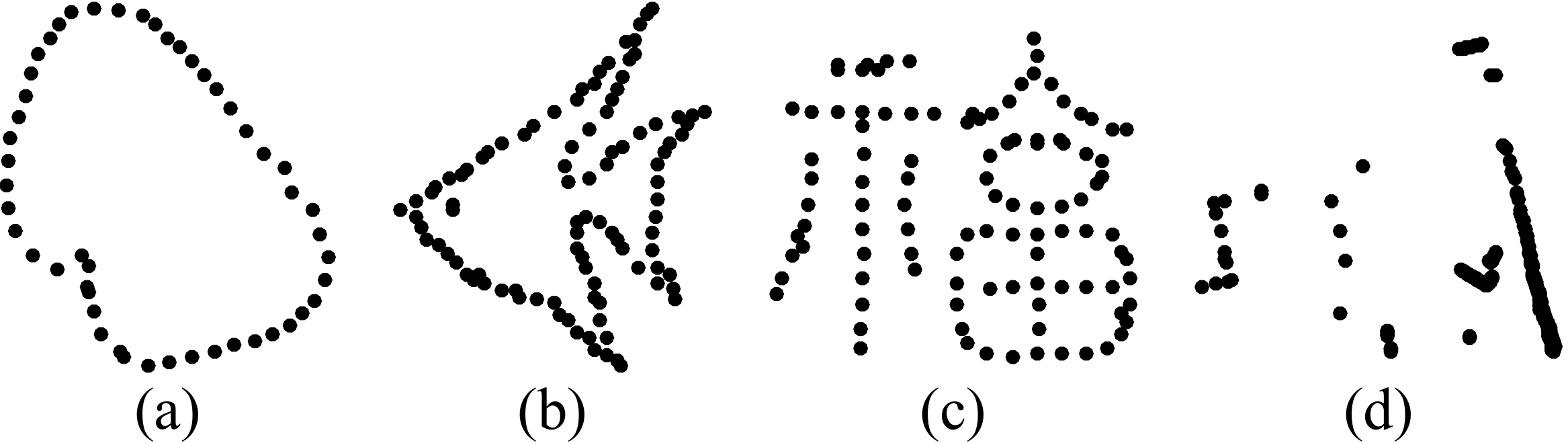}%
\caption[]{Point clouds for 2D registration experiment.}
\label{fig:model2D}
\end{figure}

\subsubsection{DO parametrization and training}
We used the same feature as that in the 3D point cloud registration experiment. For each shape, we generated $N=10000$ samples to train 30 maps as its SUM. Each sample is generated by adding the following perturbations: \emph{(1) Rotation and translation:} We added a random rotation within 85 degrees and translation within $[-0.4,0.4]^2$. \emph{(2) Noise and outliers:} We added Gaussian noise of variance of $0.03$ to each point, and added outliers of 0 to $N_M$\footnote{Recall that $N_M$ is the number of points in the model point cloud.} points in $[1.5,1.5]^2$. \emph{(3) Incomplete model:} We randomly sampled a point and removed from $0\%$ to $60\%$ of its closest points. For 2D case, we found that the shape in Fig.~\ref{fig:model2D}d had very different densities in different parts, which causes the denser area to dominate the values in $\mathbf{h}$. To alleviate this problem, in each training iteration, we preprocessed the features $\mathbf{h^{(i)}}$ by normalizing each element to lie in $[0,1]$ before learning an update map. We trained a total of $K=30$ maps, and set $\sigma^2=0.5$ and $\lambda=2\times10^{-2}$. During test, we set the maximum number of iterations to 100. Training time for DO for all shapes was less than 45 seconds, except for the shape in Fig.~\ref{fig:model2D}d which took 147 seconds due to its large number of points.

\subsubsection{Baselines and evaluation metrics}
We used the same baselines and evaluation metrics as those in the 3D point cloud registration experiment.

\begin{figure*}[!t]
\centering
\hspace{-0.5em}\includegraphics[width=6.5in]{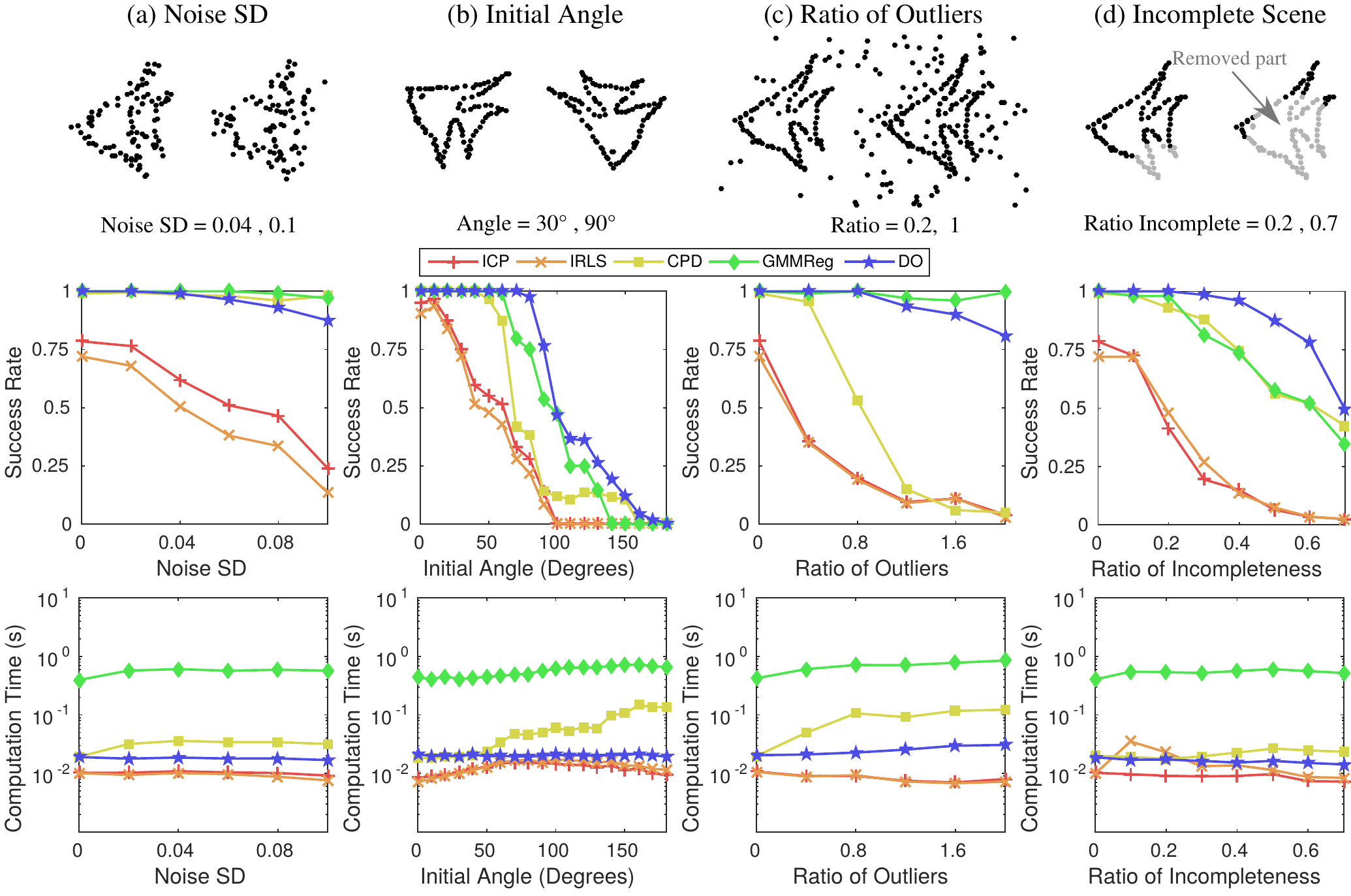}%
\caption[]{Results for 2D registration experiment. (Top) Examples of scene points with different perturbations. (Middle) Success rate. (Bottom) Computation time.}
\label{fig:synth2D}
\end{figure*}

\subsubsection{Experiments and results}
We evaluated the performance of the alignment method by varying four types of perturbations:
(1) the standard deviation of the noise ranging from 0 to 0.1, (2) the initial angle from 0 to 180 degrees, (3) the ratio of outliers from 0 to 2, and (4) the ratio of incomplete scene shape from 0 to 0.7. While we perturbed one parameter, the values of the other parameters were set as follows: noise SD = 0, initial angle uniformly sampled from 0 to 60 degrees, ratio of outliers = 0, and ratio of incompleteness = 0. For the 2D case, we did not vary the number of scene points (as in 3D case) because each point cloud was already a sparse outline of its shape. The ratio of outliers is the fraction of the number of points $N_M$ of each shape. All generated scenes included random initial translation within $[-0.4,0.4]^2$. A total of 50 rounds were run for each variable setting for each shape. 

Fig.~\ref{fig:synth2D} shows the results for the 2D registration. In terms of speed, DO performed faster than CPD and GMMReg, while being slower than ICP and IRLS. In terms of successful registration, ICP and IRLS had good success rates only when the perturbations and initial angles were small. GMMReg performed well in almost all cases, while CPD did not do well when there were a large number of outliers. DO, which learns the update steps from training data, obtained high success rates in almost all cases, but it did not perform as well as CPD and GMMReg when the noise was extremely high. This is because the noise we generated was Gaussian noise, which is the noise model assumed by both CPD and GMMReg. This shows that when the problem is accurately modelled as an optimization problem, the optimum is generally the correct solution. GMMReg also outperformed DO when the outliers were high which may be due to its annealing steps. On the other hand, DO obtained the best success rate in terms of initial angles and incomplete scenes. These perturbations cannot be easily modelled, and this result shows that it is beneficial to use learning to obtain a good solution, as done by DO. 

\section{Additional Results for Camera Pose Estimation}\label{sec:CamPoseEst}
In this section, we provide addition results for camera pose estimation experiments.

\begin{figure*}[!t]
\centering
\hspace{-0.5em}\includegraphics[width=6.5in]{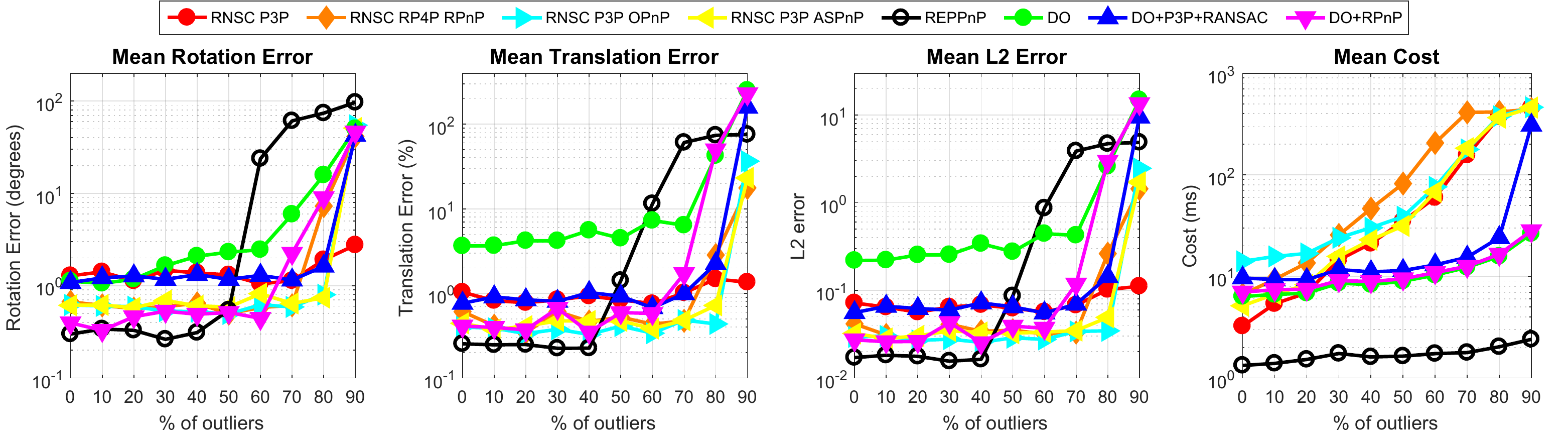}%
\caption[]{Additional results for camera pose estimation under varying percentage of outliers.}
\label{fig:resultPnP_REPPnP}
\end{figure*}

\subsection{Experiments on the code of~\cite{Ferraz_CVPR2014}}
Here, we performed additional synthetic experiments using the code of~\cite{Ferraz_CVPR2014}. The experiment varied the percentage of outliers from $0\%$~$90\%$ while fixing number of inliers at $100$ matches. For more details on the experiments and the baselines, please refer to~\cite{Ferraz_CVPR2014}.  

The result is measured in terms of the means of four metrics: (i) rotation error (in degrees), defined as $e_{\text{rot}}=\max_{k=1}^3{\text{acos}(\mathbf{r}_{k,\text{true}}^\top\mathbf{r}_k)\times 180/\pi}$, where $\mathbf{r}_{k,\text{true}}$ and $\mathbf{r}$ are the $k^{th}$ column of the ground truth and estimated rotation matrix\footnote{Note that, in the main DO paper, we measure the distance between rotation matrices using rotation angle, while $e_{\text{rot}}$ in the code of~\cite{Ferraz_CVPR2014} does not actually measure the angle. One can check this by computing $e_{\text{rot}}$ between $\mathbf{I}_3$ and a randomly generated a $180^{\circ}$ rotation matrix. If the rotation axis is not $x,y, \text{or } z$-axis, then $e_{\text{rot}}$ will not be $180^{\circ}$.}; (ii) translation error (in percentage), defined as $e_{\text{trans}}=\Vert\mathbf{t}_{true}-\mathbf{t}\Vert/\Vert\mathbf{t}\Vert\times100$, where $\mathbf{t}_{true}$ and $\mathbf{t}$ are the ground truth and estimated translation vectors, respectively; (iii) L2 error, defined as the mean $\ell_2$ distance between the 3D points in the ground truth orientation and the corresponding 3D points transformed by the estimated rotation and translation; and (iv) cost, i.e., computation time. 

Fig.~\ref{fig:resultPnP_REPPnP} shows the result. This result confirms the results in the main DO paper that DO-based approaches could be robust upto $80\%$ outliers while maintaining lower computation time than RANSAC-based approaches. On the other hand, REPPnP~\cite{Ferraz_CVPR2014} required the lowest time of all approaches while being robust upto $50\%$ outliers. Since DO main paper and~\cite{Ferraz_CVPR2014} use different data generation methods, this result shows that there can be many types of outliers, and different algorithms may be robust to one but not the others. Thus, it may not be trivial to concretely measure robustness of camera pose estimation algorithms using synthetic data, especially for non-RANSAC-based algorithms.

\subsection{Real results}
Here, we provide additional results with real images. First, we show the result on the other image from~\cite{Zheng_ICCV2013} where the object is planar in fig.~\ref{fig:real_opencv}.

\begin{figure}[!t]
\centering
\hspace{-0.5em}\includegraphics[width=3.4in]{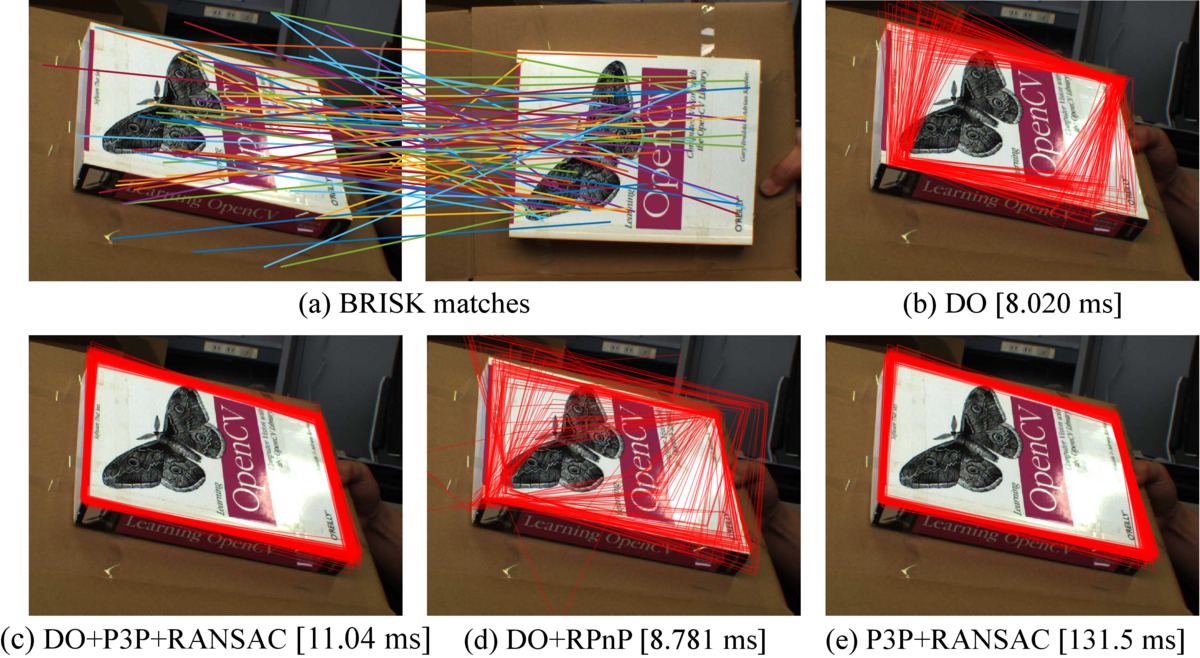}%
\caption[]{Results for camera pose estimation on real image. (a) Feature matches. Left and right images contain 2D points and projection of 3D points,~\emph{resp}. (b-d) Projected shape with average time over 100 trials.}
\label{fig:real_opencv}
\end{figure}

Next, we use the dataset from University of Oxford's Visual Geometry Group\footnote{http://www.robots.ox.ac.uk/$\sim$vgg/data/data-mview.html}~\cite{werner2002new,werner2002model} for real image experiments. The dataset contains several images with 2D points and their reconstructed 3D points, camera matrix of each image, and also reconstructed lines. Specifically, we used 'Corridor' (11 images), 'Merton College I' (3 images), 'Merton College II' (3 images), 'Merton College III' (3 images), 'University Library' (3 images), and 'Wadham College' (5 images). To perform this experiment, we selected one image from each group as a reference image (for extracting feature for the 3D point), and selected another image as the target for estimating camera pose. This results in a total of 182 pairs of images. Similar to real image experiment in the main paper, we performed 100 trials for each pair, where in each trial we transformed the 3D points with a random rotation.

Fig.~\ref{fig:oxford_result} shows the result in terms of accumulated rotation error and computation time. P3P+RANSAC obtained the best rotation accuracy since RANSAC can reliably selected the set of inliers, but it could require a long computation time to achieve this result (note that P3P+RANSAC is the fastest RANSAC-based approaches). On the other hand, DO-based approaches roughly require the same amount of time for all cases. However, DO and DO+RPnP did not obtain good rotation results. We believe this is due to the fact that the training data of DO may not well-represent the real distribution of outliers in the real-world. However, DO+P3P+RANSAC can still obtain good rotation results as the RANSAC in postprocessing step could further select a subset of good matches. To sum up, we see P3P+RANSAC obtained best rotation but it could take a long and varying computation to do so, while DO+P3P+RANSAC required roughly the same amount of time but is a little less reliable than P3P+RANSAC. This poses a trade-off between the two approaches. Future research should try to improve further on DO+P3P+RANSAC in order to improve its accuracy while maintaining low computation time.

We further provide visualization in Figs.~\ref{fig:univlib} to~\ref{fig:wadham}\footnote{We previously did not include 'Model House' as it contains pairs that do not have intersecting view of the object.}. The ground truth images were generated by projecting the 3D reconstruction of lines to the images. The results in subfigures c to f of Figs.~\ref{fig:univlib} to~\ref{fig:wadham} show the projection using the estimated camera parameters over 20 trials. We can see that most of the time DO could roughly estimate the projection, then the postprocessed results of DO+P3P+RANSAC and DO+RPnP tend to be more stable. It should be noted that even when DO did not return good result, e.g., Figs.~\ref{fig:merton3} and~\ref{fig:corridor1}, DO+P3P+RANSAC may be able to obtain a much better postprocessed results since DO may return a sufficient number of inliers for RANSAC to obtain good parameter estimate. Meanwhile, DO+RPnP cannot fix DO's result because it did not further perform sampling, and some outliers may be considered inliers, leaadting to bad estimates. There are also some cases, e.g., Fig.~\ref{fig:modelHouse2}, where all DO-based approaches failed. We believe this is due to (i) a significant large ratio of outliers, and (ii) the distribution of outliers was not well-represented in the training set. On the other hand, P3P+RANSAC obtained good estimates in all cases. In terms of time, most results of DO-based approaches terminated within 20ms, while RANSAC may use up to 200ms. 

\begin{figure}[!t]
\centering
\includegraphics[width=3.4in]{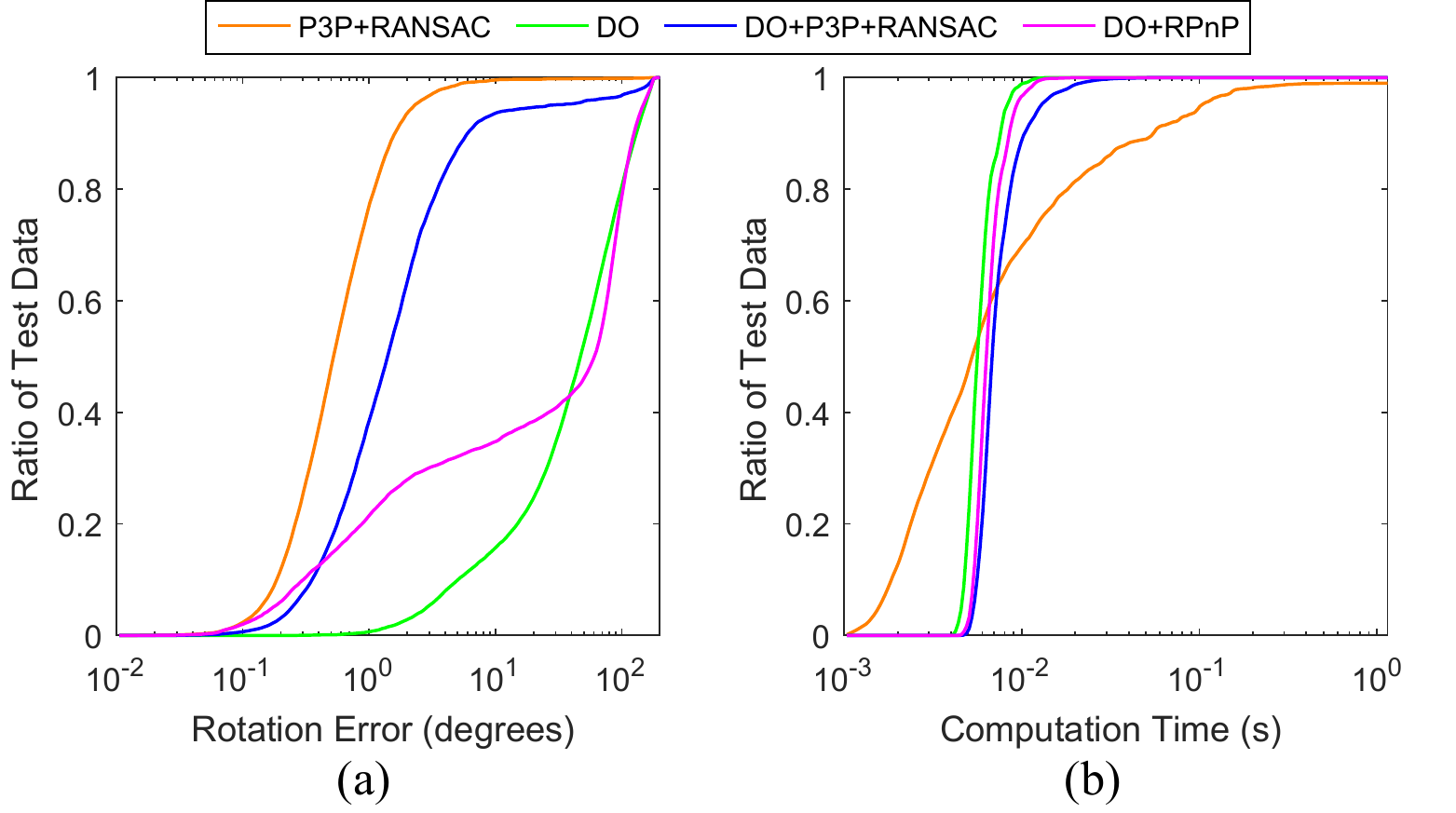}%
\caption[]{Results on the Oxford dataset in terms of (a) rotation error and (b) computation time.}
\label{fig:oxford_result}
\end{figure}

\begin{figure*}[!t]
\centering
\hspace{-0.5em}\includegraphics[width=6in]{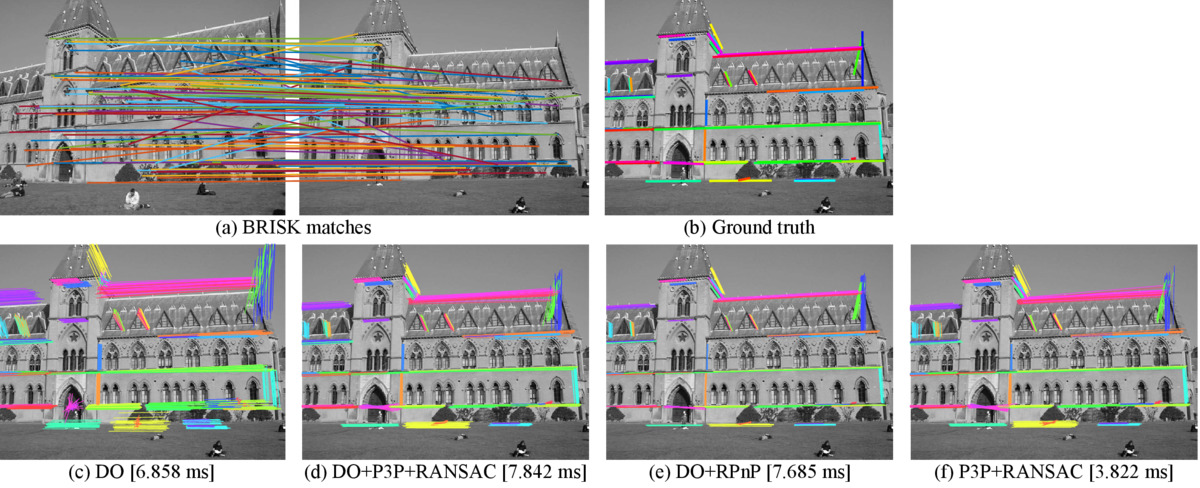}%
\caption[]{Results for camera pose estimation on real image. (a) Feature matches. Left and right images contain 2D points and projection of 3D points,~\emph{resp}. (b) Ground truth. (c-e) Projected shape with average time over 20 trials.}
\label{fig:univlib}
\end{figure*}

\begin{figure*}[!t]
\centering
\hspace{-0.5em}\includegraphics[width=6in]{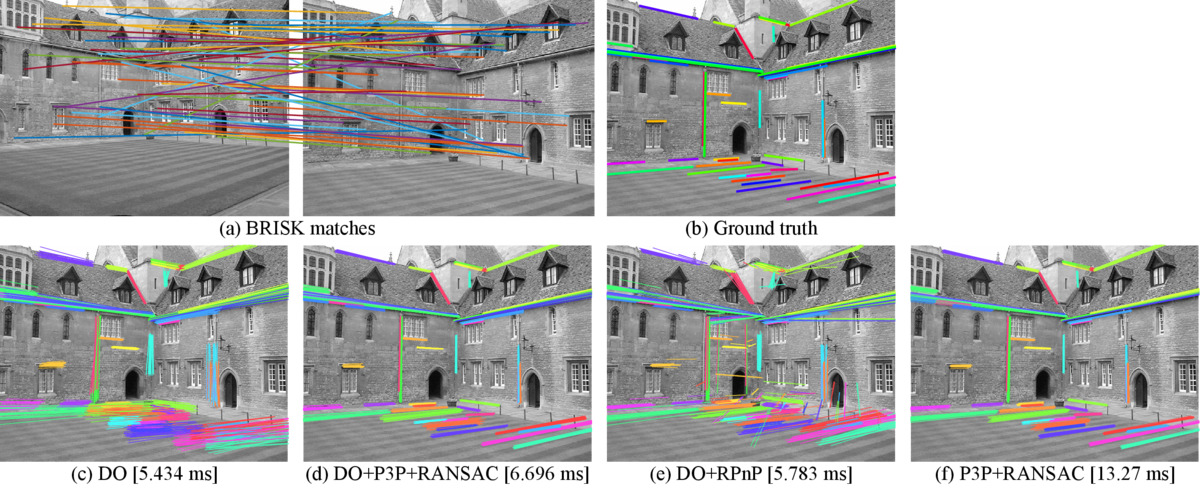}%
\caption[]{Results for camera pose estimation on real image. (a) Feature matches. Left and right images contain 2D points and projection of 3D points,~\emph{resp}. (b) Ground truth. (c-e) Projected shape with average time over 20 trials.}
\label{fig:merton1}
\end{figure*}

\begin{figure*}[!t]
\centering
\hspace{-0.5em}\includegraphics[width=6in]{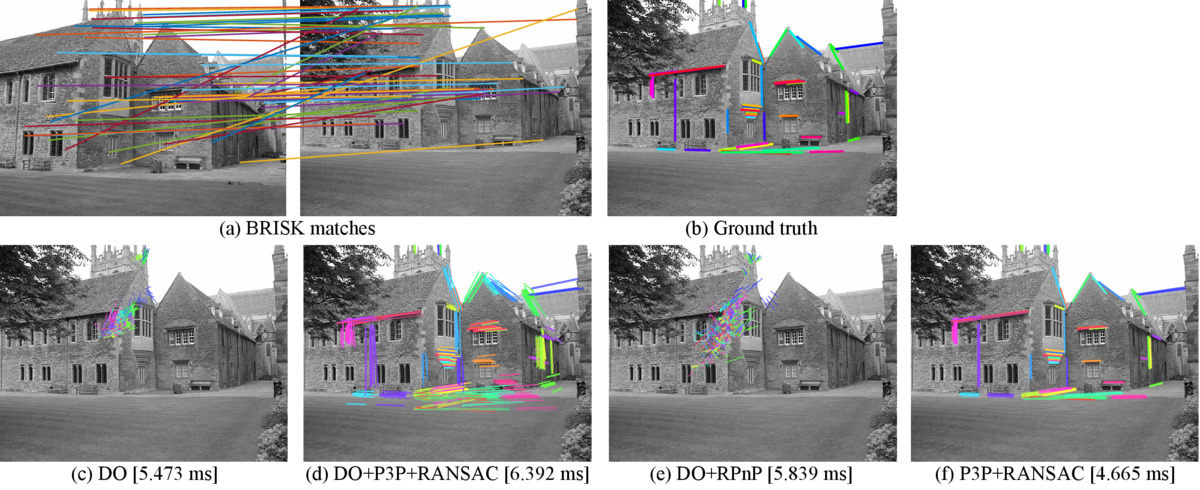}%
\caption[]{Results for camera pose estimation on real image. (a) Feature matches. Left and right images contain 2D points and projection of 3D points,~\emph{resp}. (b) Ground truth. (c-e) Projected shape with average time over 20 trials.}
\label{fig:merton3}
\end{figure*}

\begin{figure*}[!t]
\centering
\hspace{-0.5em}\includegraphics[width=6in]{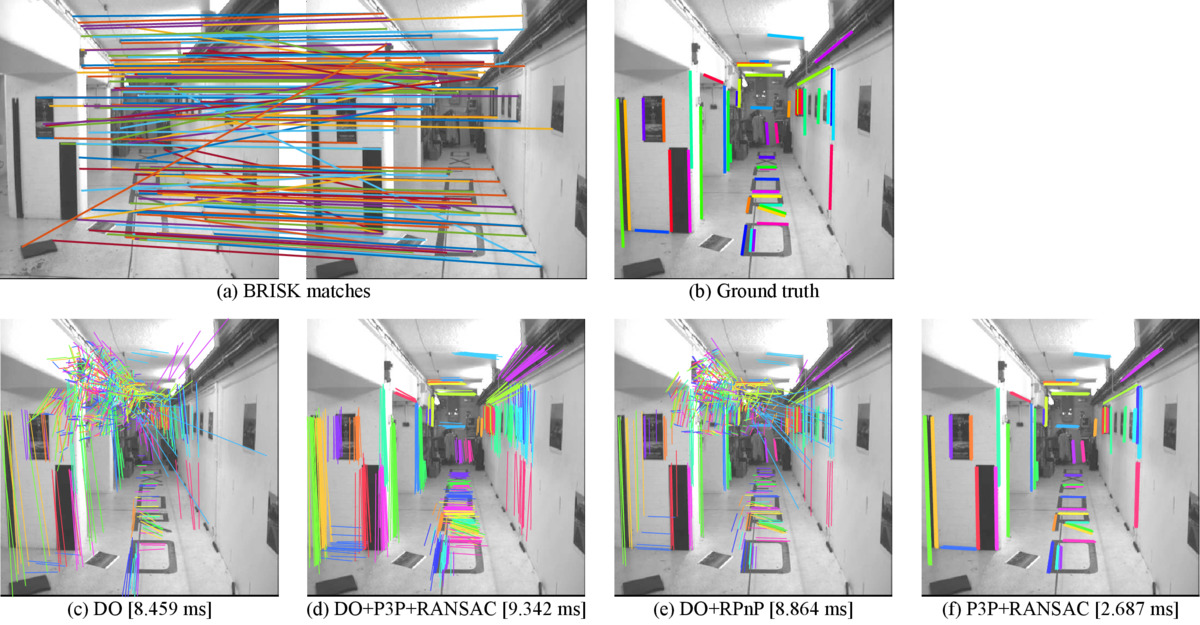}%
\caption[]{Results for camera pose estimation on real image. (a) Feature matches. Left and right images contain 2D points and projection of 3D points,~\emph{resp}. (b) Ground truth. (c-e) Projected shape with average time over 20 trials.}
\label{fig:corridor1}
\end{figure*}

\begin{figure*}[!t]
\centering
\hspace{-0.5em}\includegraphics[width=6in]{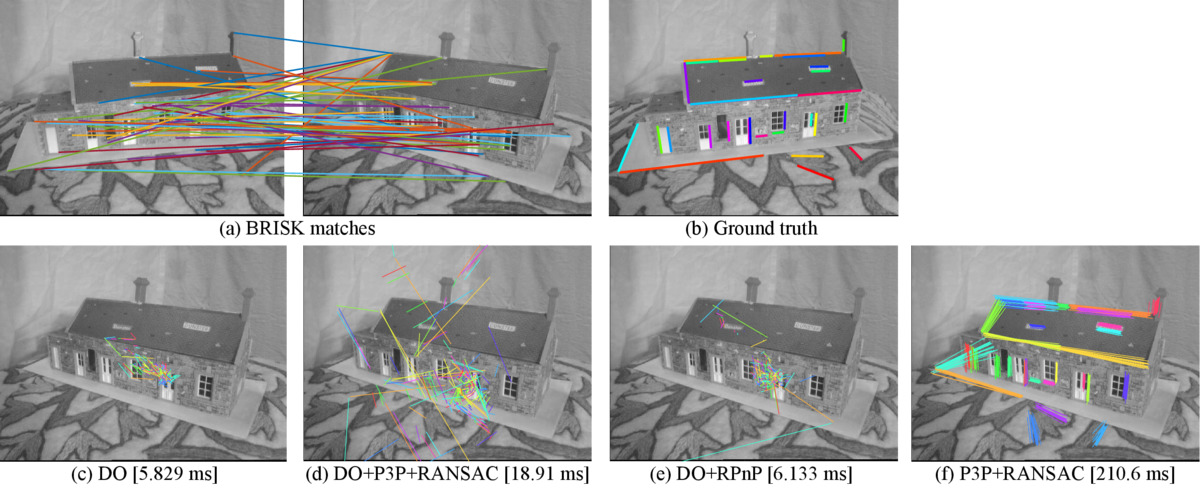}%
\caption[]{Results for camera pose estimation on real image. (a) Feature matches. Left and right images contain 2D points and projection of 3D points,~\emph{resp}. (b) Ground truth. (c-e) Projected shape with average time over 20 trials.}
\label{fig:modelHouse2}
\end{figure*}

\begin{figure*}[!t]
\centering
\hspace{-0.5em}\includegraphics[width=6in]{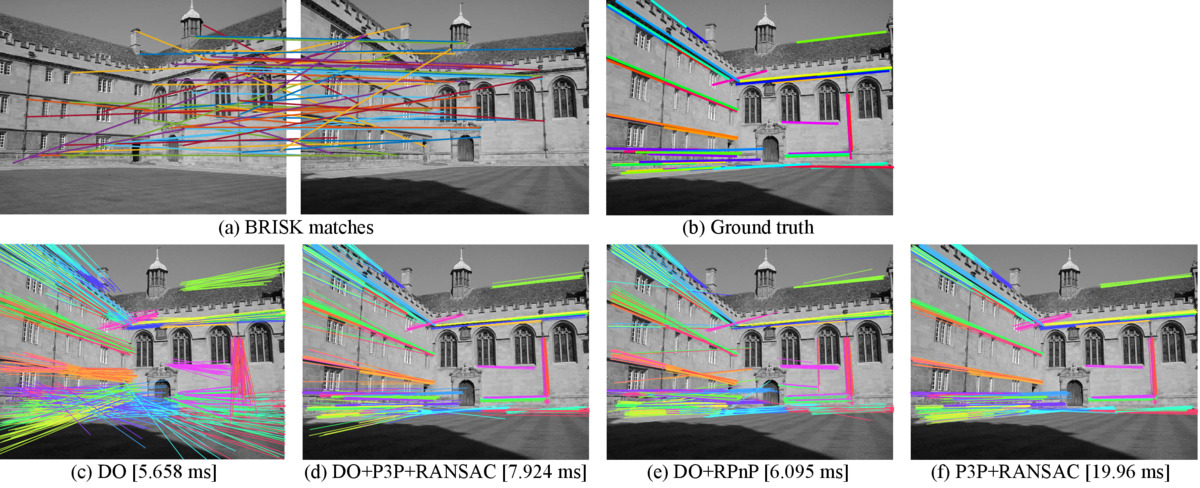}%
\caption[]{Results for camera pose estimation on real image. (a) Feature matches. Left and right images contain 2D points and projection of 3D points,~\emph{resp}. (b) Ground truth. (c-e) Projected shape with average time over 20 trials.}
\label{fig:wadham}
\end{figure*}

\section{Additional Results for Image Denoising}\label{sec:resDenoise}
Here, we provide additional visuals for image denoising experiments in Figs.~\ref{fig:denoiseAdd1} to~\ref{fig:denoiseAdd8}. PSNR is shown on the top-right of each image.

\begin{figure*}[!t]
\centering
\hspace{-0.5em}\includegraphics[width=6in]{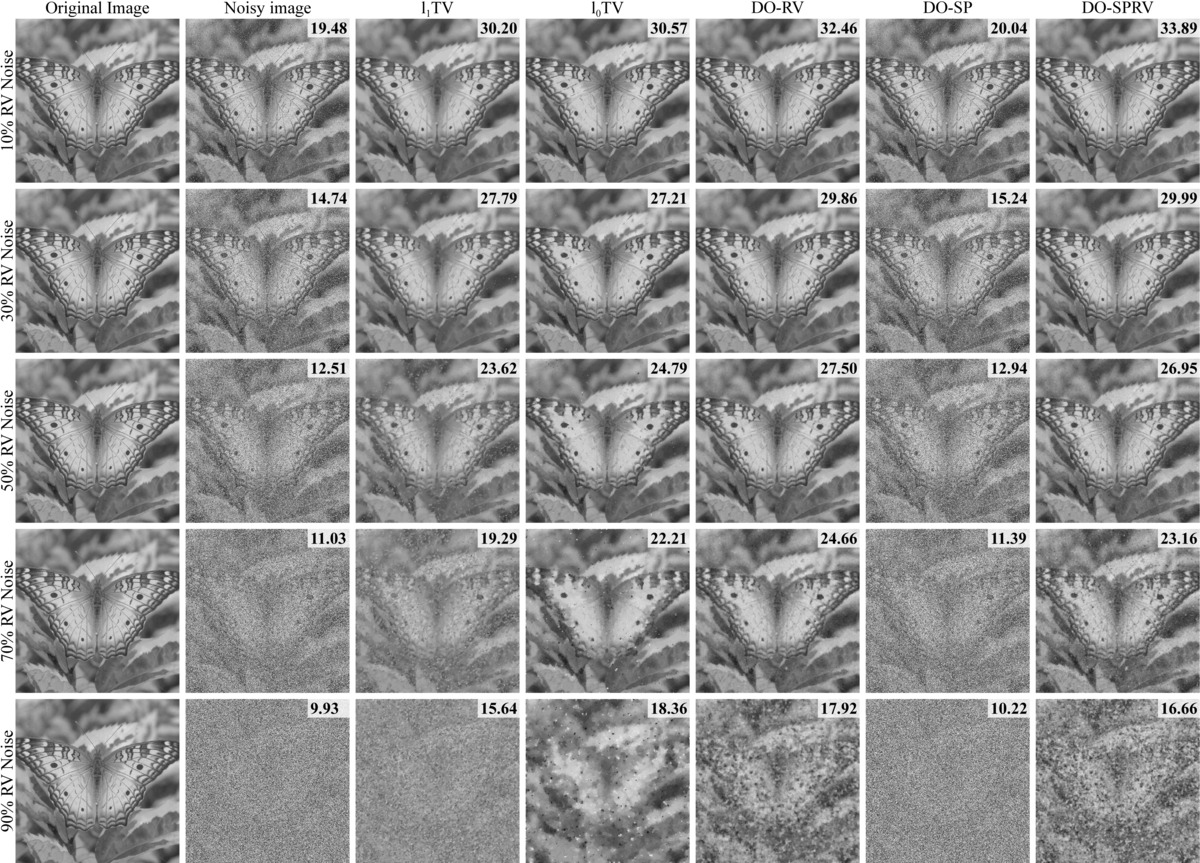}%
\caption[]{Additional results for denoising with RV noise. (Best viewed electronically)}
\label{fig:denoiseAdd1}
\end{figure*}

\begin{figure*}[!t]
\centering
\hspace{-0.5em}\includegraphics[width=6in]{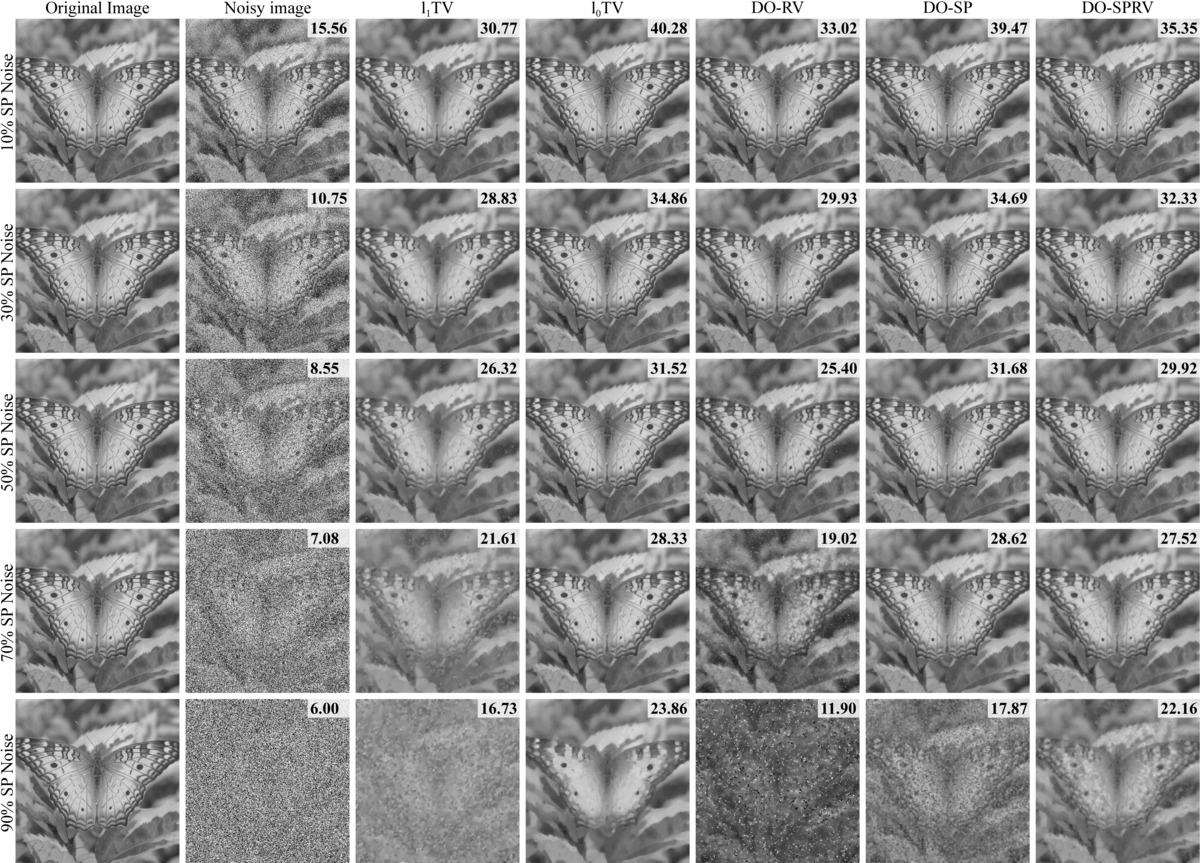}%
\caption[]{Additional results for denoising with SP noise. (Best viewed electronically)}
\label{fig:denoiseAdd2}
\end{figure*}

\begin{figure*}[!t]
\centering
\hspace{-0.5em}\includegraphics[width=6in]{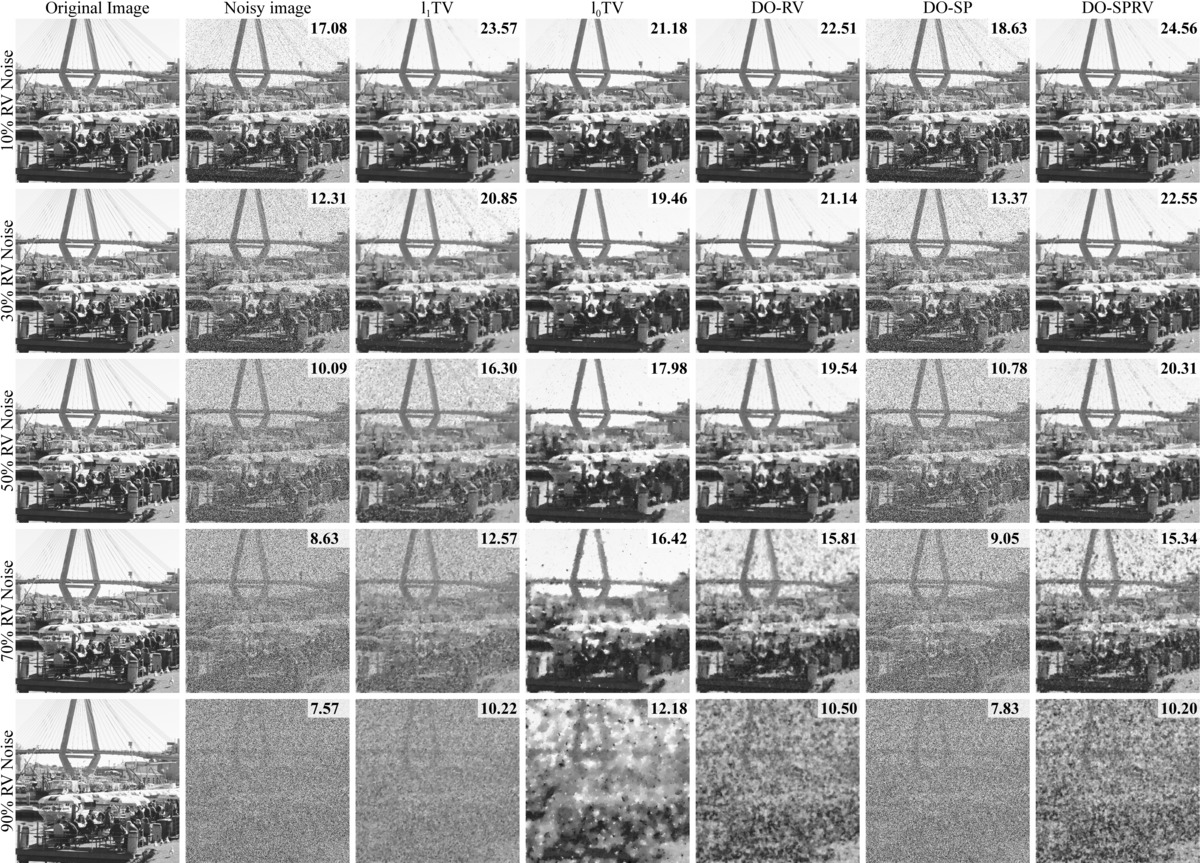}%
\caption[]{Additional results for denoising with RV noise. (Best viewed electronically)}
\label{fig:denoiseAdd3}
\end{figure*}

\begin{figure*}[!t]
\centering
\hspace{-0.5em}\includegraphics[width=6in]{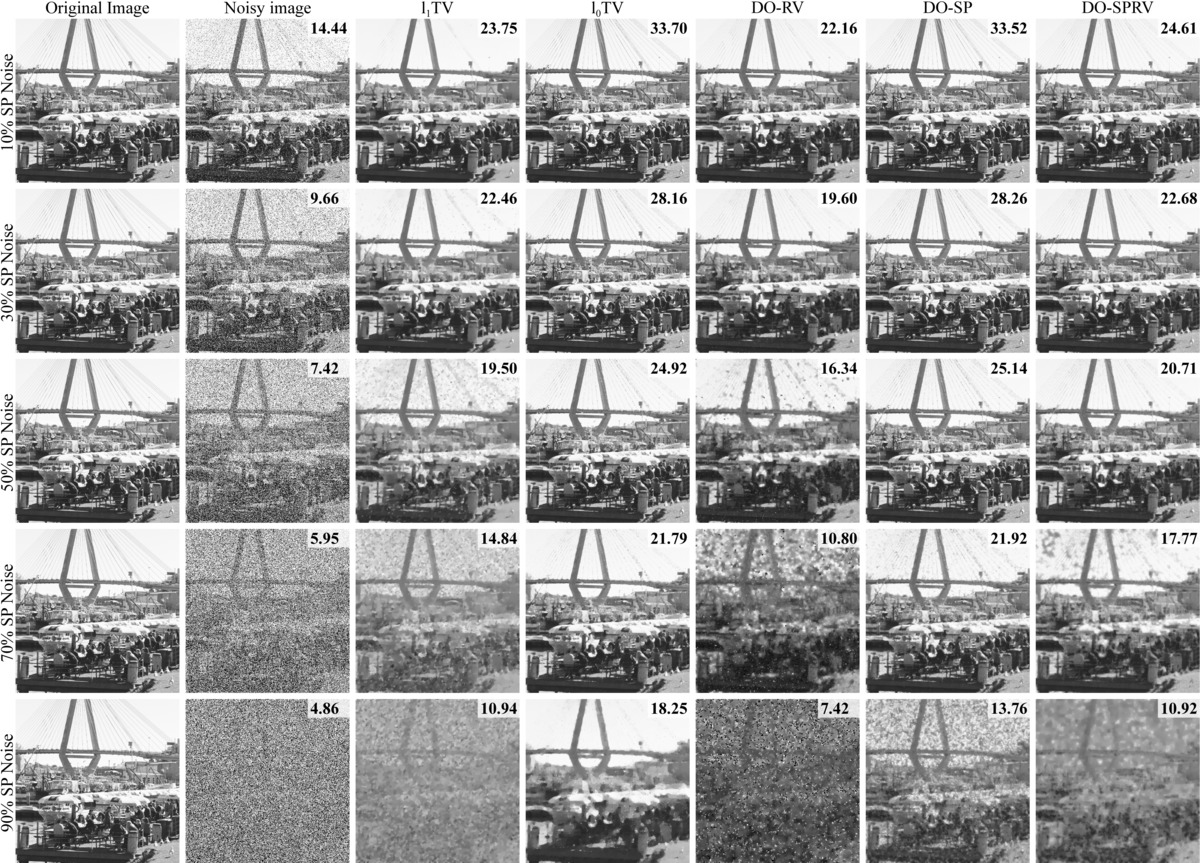}%
\caption[]{Additional results for denoising with SP noise. (Best viewed electronically)}
\label{fig:denoiseAdd4}
\end{figure*}

\begin{figure*}[!t]
\centering
\hspace{-0.5em}\includegraphics[width=6in]{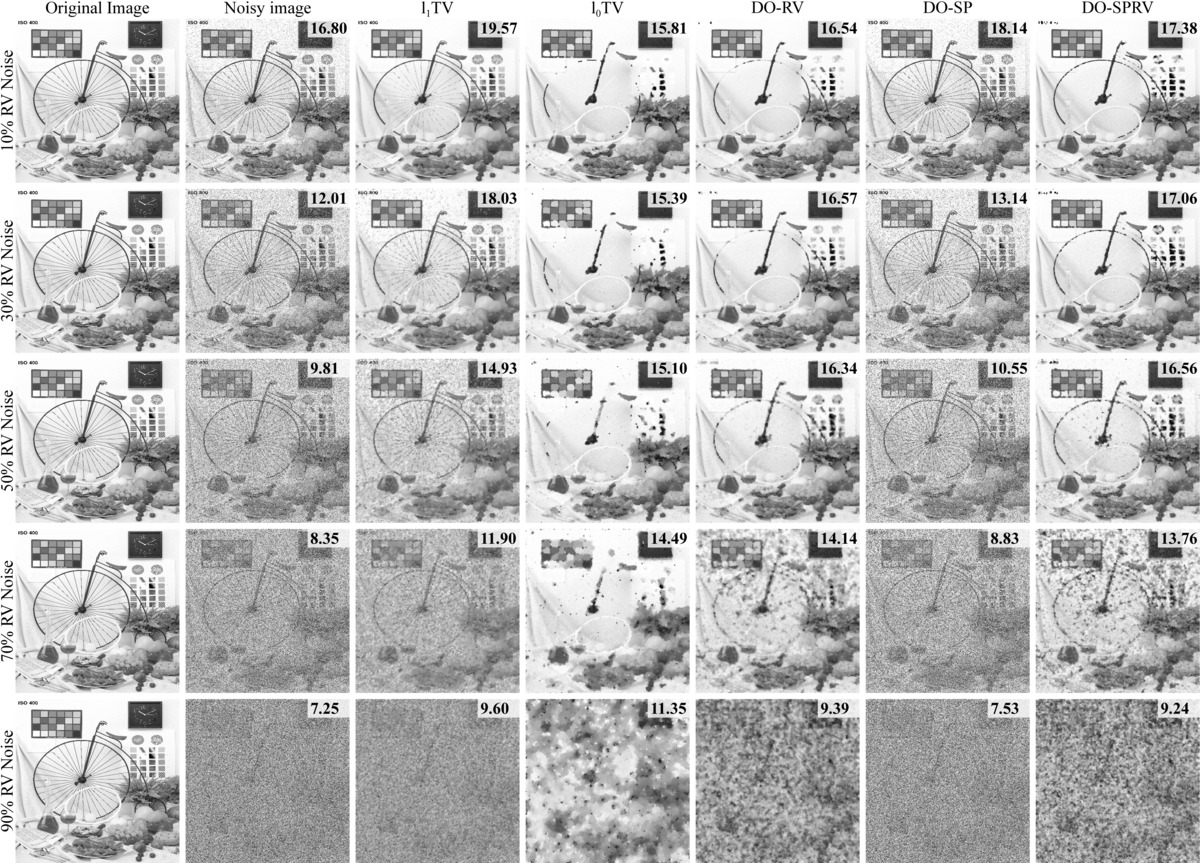}%
\caption[]{Additional results for denoising with RV noise. (Best viewed electronically)}
\label{fig:denoiseAdd5}
\end{figure*}

\begin{figure*}[!t]
\centering
\hspace{-0.5em}\includegraphics[width=6in]{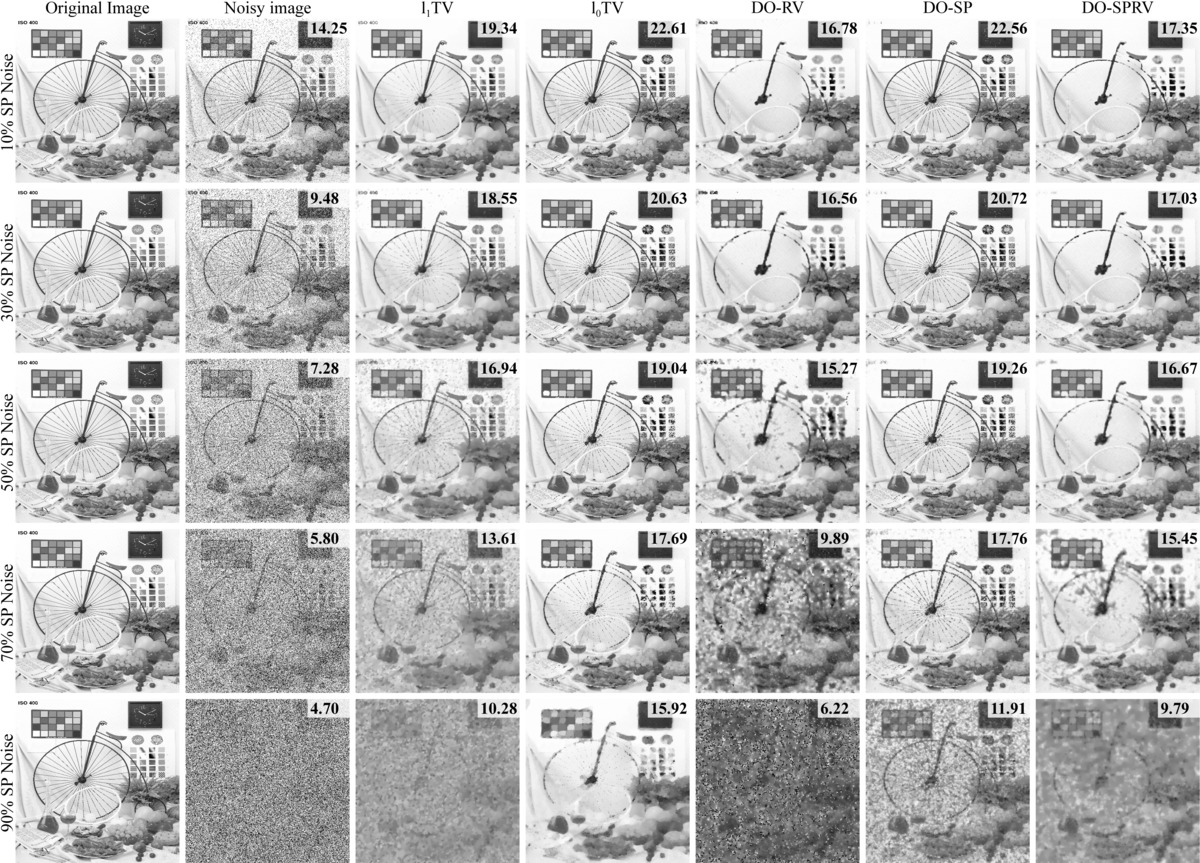}%
\caption[]{Additional results for denoising with SP noise. (Best viewed electronically)}
\label{fig:denoiseAdd6}
\end{figure*}

\begin{figure*}[!t]
\centering
\hspace{-0.5em}\includegraphics[width=6in]{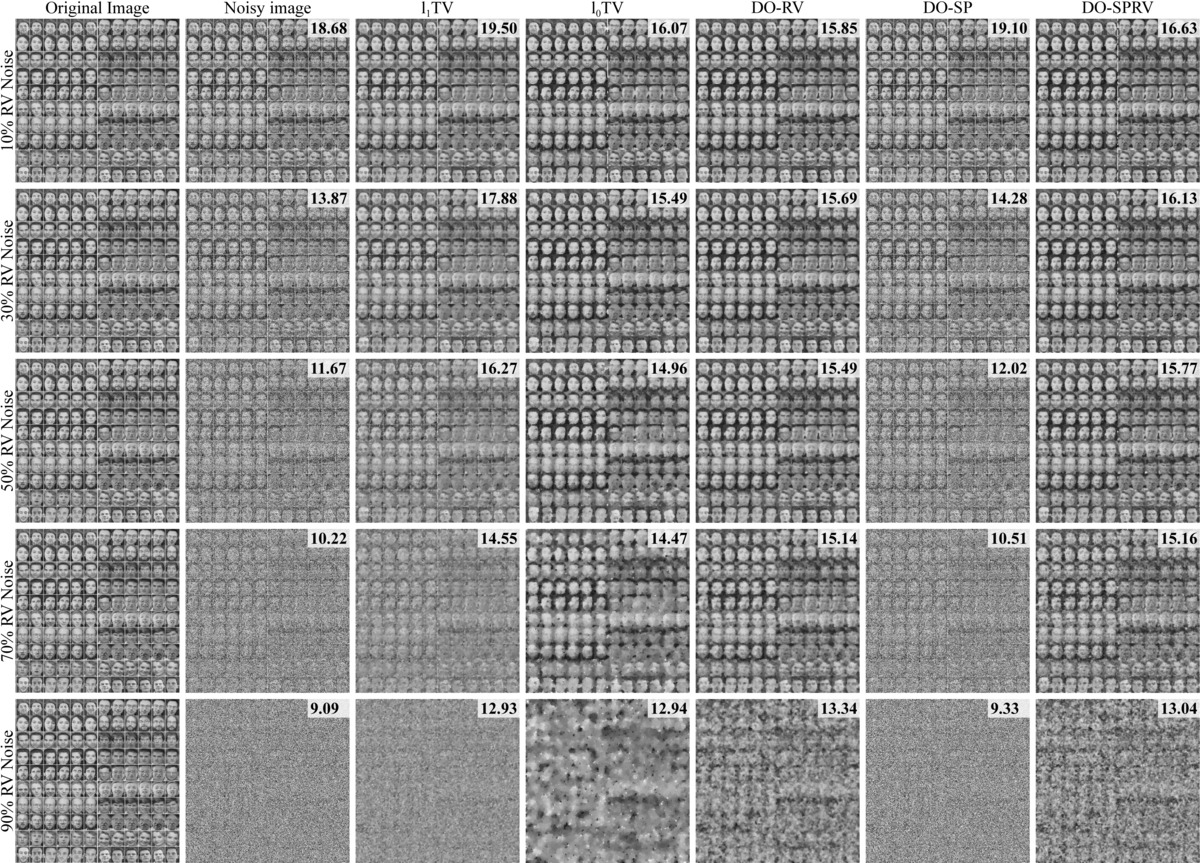}%
\caption[]{Additional results for denoising with RV noise. (Best viewed electronically)}
\label{fig:denoiseAdd7}
\end{figure*}

\begin{figure*}[!t]
\centering
\hspace{-0.5em}\includegraphics[width=6in]{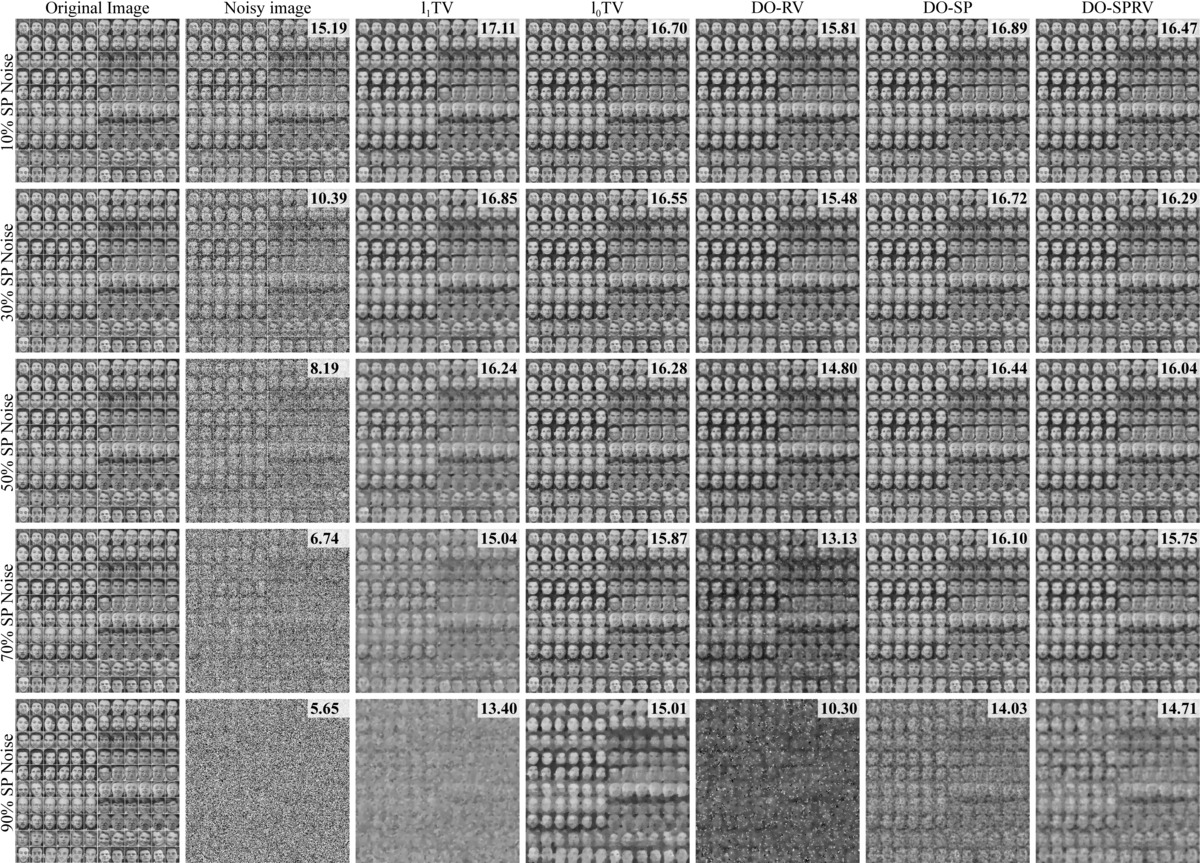}%
\caption[]{Additional results for denoising with SP noise. (Best viewed electronically)}
\label{fig:denoiseAdd8}
\end{figure*}
%
%
%
%
%
%

\end{document}